%% file: paper.tex
\documentclass{article}

\newif\ifnumericcite
\numericcitefalse

\usepackage[a4paper,margin=1in]{geometry}
\usepackage[T1]{fontenc}
\usepackage{lmodern}
\usepackage{microtype}
\usepackage{authblk}

\usepackage{url}
\usepackage{amsfonts}
\usepackage{amsmath}
\usepackage{amsthm}
\usepackage{thmtools}

\declaretheorem[
  name=Lemma,
  within=section
]{lemma}

\makeatletter
\newcommand{\mainorrestated}[2]{\ifthmt@thisistheone#1\else#2\fi}
\makeatother

\ifnumericcite
  \usepackage[numbers]{natbib}
\else
  \usepackage[round]{natbib}
\fi

\usepackage[dvipsnames]{xcolor}
\usepackage{amssymb}
\usepackage{ifthen}
\usepackage{multirow}
\usepackage{booktabs}
\usepackage{graphicx}
\usepackage{etoolbox}
\usepackage{tikz}
\usetikzlibrary{positioning,decorations.pathreplacing}
\usepackage[colorlinks=true,linkcolor=blue!60!black,citecolor=blue!60!black,urlcolor=blue!60!black]{hyperref}

\graphicspath{{../}}

\newcommand{\mainbodyfigfraction}{0.65}
\newif\ifmainbodyfloats
\mainbodyfloatstrue
\makeatletter
\g@addto@macro\@floatboxreset{%
  \ifmainbodyfloats
    \setlength{\columnwidth}{\mainbodyfigfraction\textwidth}%
    \centering
  \fi}
\let\generic@resizebox\resizebox
\newsavebox{\generic@fitbox}
\newcommand{\generic@shrinktofit}[3]{%
  \sbox{\generic@fitbox}{#3}%
  \ifdim\wd\generic@fitbox>\hsize
    \generic@resizebox{\hsize}{!}{\usebox{\generic@fitbox}}%
  \else
    \usebox{\generic@fitbox}%
  \fi}
\AtBeginEnvironment{table}{\let\resizebox\generic@shrinktofit}
\makeatother

\newboolean{showcomments}
\setboolean{showcomments}{false}
\ifthenelse{\boolean{showcomments}}
{\newcommand{\annote}[3]{{\color{#3}%
{#2}:	\color{#3}
{\em #1}}
}}{\newcommand{\annote}[3]{#1}}

\newcommand{\ph}[1]{\annote{#1}{PH}{orange}}
\newcommand{\nc}[1]{\annote{#1}{NC}{ForestGreen}}
\newcommand{\llm}[1]{\annote{#1}{LLM review}{blue}}

\title{Fast Gradient-Based Approximation of\\
Continuous Dynamic Time Warping}

\author[1]{Pierre Houedry}
\author[2]{Romain Tavenard}
\author[3]{Nicolas Courty}
\author[2]{Romaric Gaudel}
\affil[1]{INRIA Rennes}
\affil[2]{IRISA, UMR 6074, CNRS, Université de Rennes}
\affil[3]{IRISA, UMR 6074, CNRS, Université de Bretagne Sud}
\date{}

\begin{document}

\maketitle

\begin{abstract}
Dynamic Time Warping (DTW) is a common choice for measuring the discrepancy between time series. It proceeds by finding an optimal warping between timestamps through dynamic programming, which can be slow and is not differentiable. Consequently, several variants have been proposed. Among them, Soft-DTW proposes a smooth formulation of DTW, amenable to gradient computation, while Continuous-DTW (CDTW) introduces a continuous formulation through parametrizable warping functions. Yet, computational costs remain high for both methods and impractical for large datasets in modern machine learning pipelines. This paper revisits CDTW in several ways. After establishing its differentiability under mild conditions, we introduce FastCDTW, which is an efficient approximate solver with two different computational strategies: an exact closed-form scheme for piecewise-linear warping functions, and an unbiased Monte Carlo estimator for smooth ones, such as implicit neural networks. Both strategies provide gradients with respect to the warping function and the input time series, and the resulting algorithm is GPU-friendly. We demonstrate its benefits through its integration into more complex learning problems, including time series averaging, classification and deep domain adaptation.
\end{abstract}

\section{Introduction}\label{sec:introduction}
\input{sections/introduction}

\section{Background}\label{sec:background}
\input{sections/background}

\section{Properties of CDTW}\label{sec:properties}
\input{sections/properties}

\section{FastCDTW}\label{sec:method}
\input{sections/method}

\section{Numerical Experiments}\label{sec:numerical_experiments}
In the tasks covered in this section, we use FastCDTW in place of DTW or Soft-DTW in a machine learning pipeline and expect comparable accuracy, with the PL and MC schemes close to each other. Its main advantage should be speed on a GPU, most prominent when a single warping is shared across samples and learned with the model (Section~\ref{sec:complexity}).
For every experiment, the hyperparameters, hardware and remaining implementation details are given in Appendix~\ref{app:experiments}.

Before diving into the performance of FastCDTW in downstream tasks, we first illustrate the impact of series length in the computation time of DTW, Soft-DTW and FastCDTW in Figure~\ref{fig:timings}.
In the large series length regime, FastCDTW is significantly faster than DTW and Soft-DTW, with a speedup of up to 10x over Soft-DTW for series of length 10,000.
Moreover, the GPU implementation unlocks the practical use of DTW-like approaches for longer time series.

\begin{figure}[t]
    \centering
    \includegraphics[width=.8\columnwidth]{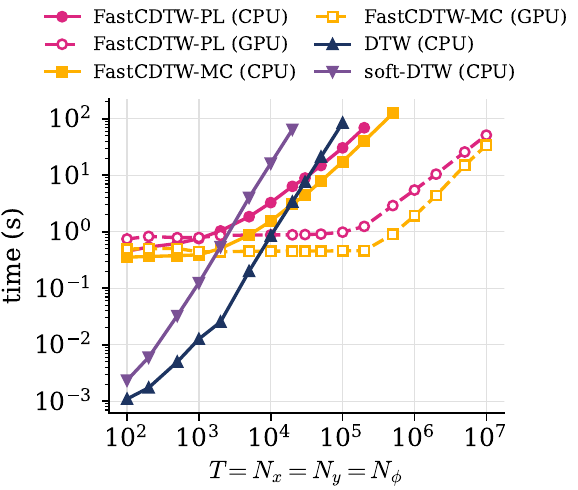}
    \caption{Computation time of DTW, Soft-DTW and FastCDTW as a function of the series length.}
    \label{fig:timings}
\end{figure}

\subsection{Time series averaging}\label{subsec:averaging}
\input{sections/experiments/averaging}

\subsection{Time series classification}\label{subsec:classification}
\input{sections/experiments/nn}

\subsection{Domain adaptation}\label{subsec:mad}
\input{sections/experiments/mad}

\section{Conclusion}\label{sec:conclusion}
\input{sections/conclusion}

\section*{Acknowledgements}
\input{sections/acknowledgements}

\ifnumericcite
  \bibliographystyle{plainnat}
\else
  \bibliographystyle{apalike}
\fi
\bibliography{biblio}

\clearpage
\mainbodyfloatsfalse
\appendix

\section{Proofs}\label{app:proofs}
\input{sections/appendix/missingproofs}

\section{Time complexity of the backward pass}\label{app:backward_cost}
\input{sections/appendix/complexity}

\section{Numerical experiments}\label{app:experiments}
\input{sections/appendix/averaging}

\input{sections/appendix/nn}

\input{sections/appendix/domain_adaptation}

\end{document}

%% file: sections/introduction.tex
Measuring similarity between time series is a core building block of many machine learning pipelines, including nearest-neighbor classification~\citep{bagnall2017bakeoff}, time series averaging and clustering~\citep{dba,softdtw}, and, more recently, domain adaptation~\citep{mad} and representation learning for temporal data~\citep{lods2017ldps,rhodes2024timewarpvaesimultaneoustimewarpingrepresentation}.
Because real-world signals rarely share a common clock (sensors sample at different rates~\citep{rubanova2019latentode}, the same phenomenon can unfold faster or slower across instances, sequences start and end at slightly different times), a similarity measure between time series must account for such temporal misalignments rather than compare observations index by index, as the Euclidean distance would do.
Dynamic Time Warping (DTW)~\citep{Vintsyuk1968,SakoeChiba1978} has become the default answer to this need: it aligns two sequences by searching, over the combinatorially many monotonic correspondences between their samples, for the one that minimizes a cumulative discrepancy.
This combinatorial construction is also the source of DTW's two main limitations.
First, DTW reasons about discrete index correspondences rather than about the underlying continuous phenomenon the sequences sample from, which makes it sensitive to the sampling rate and treats differently-sampled realizations of the same signal as fundamentally different objects.
Second, the $\arg\min$ over warping paths that defines DTW is a non-differentiable, piecewise-constant function of its inputs, which rules out the use of DTW as a training loss or as a component of a pipeline optimized by gradient descent.

These two limitations have each spawned their own line of remedies, but, as we argue below, neither has been resolved without giving up on the other.
On the differentiability side, \citet{softdtw} introduced Soft-DTW, which replaces the minimum over alignments with a soft-minimum controlled by a temperature parameter.
This relaxation is genuinely differentiable and has enabled applications such as barycenter computation, clustering, and end-to-end learning with temporal alignment.
The relaxation, however, is not without cost: smoothing the minimum over an exponential number of paths tends to \emph{average} alignments rather than \emph{select} among them, which in practice manifests as a systematic denoising effect~\citep{softdtwdiv}.

A separate line of work reformulates time warping directly in continuous time.
Rather than an index correspondence, a warping is modeled as a pair of continuous, monotonic reparametrizations of the two signals, and the cost of the alignment is obtained by \emph{integrating}, rather than summing, a local discrepancy along the corresponding curves.
This continuous counterpart to DTW, referred to as Continuous DTW (CDTW)~\citep{buchin2025cdtw}, has been argued to avoid artifacts that plague DTW.
Turning this appealing formulation into a usable algorithm, however, has proven difficult: existing methods either operate at a computational cost that remains impractical at the sequence lengths and iteration budgets required by machine learning pipelines, and without differentiability.
Section~\ref{sec:background} reviews these methods and their complexity in detail.

Soft-DTW and CDTW thus address complementary symptoms of the same underlying issue (respectively, the non-differentiability of DTW and its discreteness), but neither closes both gaps at once, and the continuous formulations that do exist remain, by their own authors' admission, computationally out of reach for practical use.
What is missing is a similarity measure that is simultaneously (i) genuinely continuous, to avoid the sampling-rate artifacts of discrete alignments; (ii) differentiable, so that it can serve as a training loss or as a component of a larger learned pipeline and (iii) tractable at the sequence lengths and iteration counts encountered in practice.

We introduce \emph{FastCDTW} which proceeds by solving the CDTW problem by gradient descent over a finite-dimensional parametrization of the warping functions rather than solving this problem to global optimality via dynamic programming. %
This choice trades the global-optimality guarantee enjoyed by exact CDTW solvers for tractability: computing FastCDTW costs orders of magnitude less than exact or approximate CDTW solvers, and it is naturally suited to batched GPU computation.
The price to pay is that gradient descent on the resulting non-convex objective is not guaranteed to converge to the global minimum.
We show throughout Section~\ref{sec:method} and in our experiments that, despite this lack of global guarantees, the resulting approximation is accurate enough, and considerably faster, to be useful in practice.

We make the following contributions:
\begin{itemize}
    \item \textbf{Theoretical properties of CDTW} (Section~\ref{sec:properties}): we reduce the CDTW optimization to a single warping, show that Lipschitz warpings reach the same optimal cost, and prove the differentiability of CDTW with respect to an input curve when the warping is restricted to a parametrized family. To our knowledge, the last two results are new;
    \item \textbf{FastCDTW}, a gradient-based method to compute CDTW (Section~\ref{sec:method}): the warping is optimized by gradient descent over a finite-dimensional parametrization, and the cost is evaluated either in closed form for piecewise-linear warpings or by an unbiased Monte Carlo estimator for smooth ones. The method is GPU-friendly;
    \item \textbf{An empirical validation} on averaging, classification and domain adaptation tasks, showing that FastCDTW can be used as a replacement for DTW or Soft-DTW at a lower computational cost: it notably avoids the peak-erosion pathology of Soft-DTW barycenters, and runs up to $50\times$ faster than DTW when plugged into a domain adaptation pipeline. %

\end{itemize}

%% file: sections/background.tex
All the alignment metrics discussed in this section can be defined with any discrepancy between samples: for instance, \citet{buchin2025cdtw} compute CDTW between one-dimensional curves with the absolute difference between points, and measure the length of the warping path with the $L_1$ norm.
Throughout, we use the squared Euclidean distance, one of the most common choices, and the Euclidean length of the warping path.

\subsection{Dynamic Time Warping}

Dynamic Time Warping (DTW) is an alignment-based similarity measure for time series that was introduced independently by \citet{Vintsyuk1968} and \citet{SakoeChiba1978} for speech recognition applications.
Unlike the Euclidean distance, which matches observations strictly by their timestamps, DTW finds an optimal warping path through the cost matrix that minimizes the cumulative distance between two sequences, thereby accounting for local speed variations in the time axis.
Given two time series
\[
x = (x_0, \dots, x_{N_x-1}),\quad y = (y_0, \dots, y_{N_y-1})\,,
\]
with each $x_i, y_j\in\mathbb{R}^d$, DTW is defined as
\begin{equation}\label{eq:dtw}
  \mathrm{DTW}(x,y) \;:=\;\min_{\pi\in\mathcal A(x,y)}\;\Bigl(\sum_{(i,j)\in\pi} \lVert x_i-y_j\rVert_2^2\Bigr)^{1/2},
\end{equation}
where $\mathcal A(x,y)$ denotes the set of admissible warping paths from $(0,0)$ to $(N_x-1,N_y-1)$ satisfying the usual monotonicity and step-size constraints.
Although the number of admissible paths grows combinatorially, the optimal warping cost can be computed exactly via dynamic programming in $O(N_xN_yd)$ time.

However, the optimal warping path is a non-smooth function of the inputs, limiting its direct use in gradient-based optimization.
We therefore review two lines of work that relax different aspects of DTW: differentiability and its discrete, index-based alignment.

\subsection{Differentiable relaxations of DTW}

Soft-DTW~\citep{softdtw} replaces the minimum over admissible paths in Equation~\eqref{eq:dtw} by a soft-minimum, so that the resulting cost function $\mathrm{SoftDTW}_\gamma(x,y)$ reads as 
\[
-\gamma \log \sum_{\pi\in\mathcal A(x,y)} \exp \Bigl(-\frac{1}{\gamma}\sum_{(i,j)\in\pi} \lVert x_i-y_j\rVert_2^2\Bigr)
\]
with $\gamma > 0$ controlling the trade-off between smoothness and fidelity to the original DTW cost.
This cost can be computed with the same dynamic-programming recursion as DTW, replacing the hard minimum by a soft-minimum.
It is differentiable, with a gradient computable in $O(N_xN_yd)$ time and space via a backward recursion, which has made Soft-DTW the common choice for using DTW-like alignments as a training loss or inside a barycenter/clustering objective.
This smoothing has a side effect: rather than selecting the best-aligned path, the soft-minimum blends many competing alignments, which tends to blur sharp, localized features~\cite{softdtwdiv}.
Soft-DTW thus trades DTW's non-differentiability for a smoothing bias while keeping a discrete, index-based alignment.

\subsection{Continuous-time formulations}

A separate line of work instead formulates time warping directly in continuous time, replacing the warping path $\pi$ with a pair of continuous, monotonic reparametrizations of the two signals. %
This continuous counterpart to DTW is generally referred to as Continuous DTW (CDTW)~\citep{buchin2007frechet}.
More precisely, given two curves $x, y : [0,1] \rightarrow \mathbb{R}^d$, the CDTW functional is defined as:
\begin{equation}\label{eq:cdtw}
{\text{CDTW}}(x,y):= \min_{\phi, \psi \in \mathcal{C}}\mathcal{L}(\phi,\psi),
\end{equation}
where the alignment cost of a pair of warping maps is
\begin{equation}\label{eq:cost}
\mathcal{L} (\phi,\psi) := \int_{0}^{1}\bigl\lVert x\bigl(\phi(t)\bigr) - y\bigl(\psi(t)\bigr)\bigr\rVert_2^2
\bigl\lVert (\dot\phi(t), \dot\psi(t))\bigr\rVert_2
\;\mathrm{d}t,
\end{equation}
the dependence on the two curves $x$ and $y$ being left implicit and
with $\mathcal{C}$ being the set of increasing $C^1$-diffeomorphisms $\phi:[0,1] \rightarrow [0,1]$ satisfying $\phi(0)=0$ and $\phi(1)=1$.
Beyond the switch from a sum to an integral, the multiplicative factor $\lVert(\dot\phi,\dot\psi)\rVert_2$ is another key difference with DTW. %
Geometrically, $(\phi,\psi)$ is a path from $(0,0)$ to $(1,1)$ in the unit square and this factor is its length element, so $\mathcal{L}(\phi,\psi)$ is the length of the path weighted at each point $(u,v)$ by $\lVert x(u)-y(v)\rVert_2^2$~\citep{buchin2007frechet}. The cost therefore does not depend on the speed at which the path is traversed, which rules out uninformative solutions and allows reducing CDTW to a single warping (Proposition~\ref{prop:reparametrization}).

Computing CDTW is itself a hard problem.
\citet{buchin2025cdtw} proposed the first exact algorithm for CDTW, propagating continuous functions through a dynamic program that runs in $O(N_x^5)$ time but only for one-dimensional curves of at most $N_x$ samples each.
\citet{ergur2026cdtw} instead target an approximate solution, proposing a fully polynomial-time approximation scheme (FPTAS) of multiplicative error $\varepsilon$ for CDTW between piecewise algebraic curves of bounded degree, with complexity $O\bigl((N_x+N_y)^{19/6}(1/\varepsilon)^{10/3}\log((N_x+N_y)/\varepsilon^2)\bigr)$.
Both algorithms are restricted to specific curve representations (one-dimensional, or piecewise algebraic) rather than arbitrary sampled, multivariate time series, and neither is designed to produce a differentiable objective: they compute a single scalar distance value through a dynamic program, with no notion of a gradient with respect to the input signals.

GDTW~\citep{gdtw} takes a slightly different route: it warps only one of the signals, and replaces the multiplicative factor of Equation~\eqref{eq:cost} by adding two regularization terms, penalizing deviations of the warping $\phi$ from the identity and of its derivative from $1$.
These two regularizers play the same role as the multiplicative factor $\lVert(\dot\phi,\dot\psi)\rVert_2$ in CDTW but at the price of tuning two hyperparameters weighting each of the two terms.
Despite its continuous-time formulation, the problem above is solved by discretizing the space of warping values and running a dynamic program to find the optimal warping on this discretized grid, iteratively refining the grid to reach high accuracy.
As a result, GDTW remains, at its core, a dynamic-programming solver: it is poorly suited for autodifferentiation or to batched GPU computation, which limits its use as a building block inside larger, gradient-trained pipelines.

Taken together, these two lines of work each relax one of DTW's limitations while leaving the other largely untouched: Soft-DTW is differentiable but keeps the discrete alignment, while GDTW and CDTW solvers reformulate time warping in continuous time but remain fundamentally dynamic-programming algorithms, unfit for embedding inside models trained end to end. %
This is the gap our solver strategy, introduced in Section~\ref{sec:method}, aims to close.

%% file: sections/properties.tex
Before turning to its computation, we derive some  properties of the CDTW objective~\eqref{eq:cost}: the reduction to a single warping (Proposition~\ref{prop:reparametrization}), the fact that Lipschitz warpings reach the same optimal cost (Proposition~\ref{prop:grid-approx}), and the differentiability, with respect to an input curve, of CDTW restricted to a parametrized family of warpings (Proposition~\ref{prop:danskin}).
A similar reduction to a single path appears in~\citet[Lemma~1]{buchin2025cdtw}, but only for one-dimensional curves. To the best of our knowledge, Propositions~\ref{prop:grid-approx} and~\ref{prop:danskin} are new. All proofs are given in Appendix~\ref{app:proofs}.

\paragraph{Reparametrization.}
A property of the functional $\mathcal{L}$ is that the optimization can always be reduced to a single reparametrization,
\begin{restatable}{Proposition}{reparametrization}
\label{prop:reparametrization}
Let \(x,y:[0,1]\to\mathbb{R}^d\) be two continuously differentiable curves and let \(\mathcal{L}\) be as in \eqref{eq:cost}. Then for every $\phi, \psi \in \mathcal{C}$, there exists a reparametrization \(\tilde{\phi} \in \mathcal{C}\) of $x$ such that
\[
\mathcal{L}(\phi,\psi)
=\mathcal{L}\bigl(\tilde\phi,\mathrm{Id}\bigr),
\]
where
$
\tilde\phi \;=\; \phi\circ\psi^{-1}
\text{ and }
\mathrm{Id}(u)=u.
$
\end{restatable}

From now on we therefore optimize over a single warping map $\phi$, and write $\mathcal{L}(\phi):=\mathcal{L}(\phi,\mathrm{Id})$.

\paragraph{Lipschitz maps.}
The set $\mathcal C$ in~\eqref{eq:cdtw} only contains smooth, strictly increasing warpings, whereas practical warpings such as piecewise-linear ones can have corners and flat pieces. The following proposition shows that allowing such warpings does not decrease the optimal cost.

\begin{restatable}{Proposition}{gridapprox}
\label{prop:grid-approx}
Let $x,y:[0,1]\to\mathbb R^d$ be continuous, and let $\mathcal C_{\mathrm{Lip}}$ be the set of nondecreasing Lipschitz maps $\phi:[0,1]\to[0,1]$ with $\phi(0)=0$ and $\phi(1)=1$. Then
\[
\inf_{\phi\in\mathcal C_{\mathrm{Lip}}}\mathcal L(\phi)=\inf_{\phi\in\mathcal C}\mathcal L(\phi)=\mathrm{CDTW}(x,y).
\]
\end{restatable}

\paragraph{Differentiability.}
CDTW can be used as a training loss on a curve $y$ produced by a network. The variable to optimize is then $y\in \mathcal{Y}:= L^2([0,1],\mathbb R^d)$ and we write $\mathcal L_y(\phi)$ to make its dependence on $y$ explicit. In practice, $\phi$ is not optimized over all of $\mathcal C_{\mathrm{Lip}}$ but over a family $\{\phi_\theta:\theta\in\Theta \subset\mathbb R^p\}$ indexed by a finite-dimensional parameter, the family used in this paper is introduced in Section~\ref{sec:param}. The quantity to differentiate is then $\mathrm{CDTW}_\Theta(x,y):=\inf_{\theta\in\Theta}\mathcal L_y(\phi_\theta)$, itself defined by a minimization over $\theta$.
The following proposition states that it is enough to differentiate $\mathcal L_y$ at an optimal warping, held fixed.

\begin{restatable}{Proposition}{danskin}
\label{prop:danskin}
\mainorrestated{%
Assume that $x$ is continuous on $[0,1]$, that $\Theta\subset\mathbb R^p$ is compact, that for every $\theta\in\Theta$ the map $\phi_\theta:[0,1]\to[0,1]$ is nondecreasing and Lipschitz with $\phi_\theta(0)=0$ and $\phi_\theta(1)=1$, and that $\theta\mapsto\dot\phi_\theta$ is continuous from $\Theta$ to $L^\infty([0,1])$. Let $y\in\mathcal Y$. Then:%
}{%
Assume that
\begin{itemize}
\item[(A1)] $x$ is continuous on $[0,1]$;
\item[(A2)] $\Theta\subset\mathbb R^p$ is compact;
\item[(A3)] for every $\theta\in\Theta$, $\phi_\theta:[0,1]\to[0,1]$ is nondecreasing and Lipschitz, with $\phi_\theta(0)=0$ and $\phi_\theta(1)=1$;
\item[(A4)] $\theta\mapsto\dot\phi_\theta$ is continuous from $\Theta$ to $L^\infty([0,1])$,
\end{itemize}
and let $y\in\mathcal Y$. Then:%
}
\begin{itemize}
\item[(i)] the set $\Phi(y):=\arg\min_{\theta\in\Theta}\mathcal L_y(\phi_\theta)$ is nonempty and compact, so the infimum defining $\mathrm{CDTW}_\Theta(x,y)$ is a minimum;
\item[(ii)] for every $h\in\mathcal Y$, $y\mapsto\mathrm{CDTW}_\Theta(x,y)$ has at $y$ a one-sided directional derivative in the direction $h$, given by
\begin{equation}\label{eq:fmtw-gradient-y}
\partial_h\,\mathrm{CDTW}_\Theta(x,y)=\min_{\theta\in\Phi(y)}\partial_h\mathcal L_y(\phi_\theta),
\end{equation}
where, for every Lipschitz $\phi:[0,1]\to[0,1]$,
\begin{equation}\label{eq:dir-deriv-explicit}
\partial_h\mathcal L_y(\phi)=-2\int_0^1\sqrt{\dot\phi^2+1}\,\bigl\langle x\circ\phi-y,\,h\bigr\rangle\,\mathrm dt .
\end{equation}
\end{itemize}
\end{restatable}

%% file: sections/method.tex
Since the warping $\phi$ above is a function rather than a sequence of index pairs, it can be optimized by gradient descent instead of dynamic programming, provided the cost of a candidate $\phi$ can be evaluated cheaply and differentiated. Section~\ref{sec:methodology} gives two ways of evaluating this cost.
Unlike the exact CDTW solver of Section~\ref{sec:background}, we do not claim to return the global optimum: we return the value at the warping our optimizer reaches, for a cost that makes it usable inside a model trained end to end.

\subsection{Evaluating the alignment cost}\label{sec:methodology}

We propose two schemes to evaluate $\mathcal{L}(\phi)$.
First, the \emph{Piecewise-linear} scheme, which assumes $\phi$, $x$ and $y$ are piecewise-linear functions of time, evaluates $\mathcal{L}(\phi)$ in closed form, without any discretization error (Proposition~\ref{prop:pl-closed-form}).
Second, the \emph{Monte-Carlo} scheme, which only assumes $\phi$ and $\dot\phi$ can be evaluated pointwise, estimates $\mathcal{L}(\phi)$ by Monte Carlo sampling (Equation~\eqref{eq:mc-estimator}).

\paragraph{Piecewise-Linear (PL) scheme.}
In most real-world applications we do \emph{not} have access to the exact continuous maps $x$ and $y$: we are only given samples
\[
\{(t_{x,i},\,x_i)\}_{i=0}^{N_x-1},\qquad
\{(t_{y,j},\,y_j)\}_{j=0}^{N_y-1}.
\]
We denote by $x_\mathrm{PL}$ and $y_\mathrm{PL}$ their piecewise-linear interpolants, and by $\phi_{x,\mathrm{PL}}$ the piecewise-linear map defined by values $\{\phi_{x,k}\}_{k=0}^{N_\phi}$ on the uniform grid $t_k=k/N_\phi$.

In this case, the integral of $\mathcal{L}(\phi)$ splits over intervals and each piece integrates in closed form.

\begin{restatable}{Proposition}{plclosedform}
\label{prop:pl-closed-form}
Assume that $\phi_{x,\mathrm{PL}}$ is strictly increasing, so that each preimage $\phi_{x,\mathrm{PL}}^{-1}(t_{x,i})$ is a single point.
Let $\{\tau_k\}_{k=0}^{M}$ be the strictly increasing union of
\[
\bigl\{\phi_{x,\mathrm{PL}}^{-1}(t_{x,i})\bigr\}_{i=0}^{N_x-1},\quad
\bigl\{t_{y,j}\bigr\}_{j=0}^{N_y-1},\quad
\bigl\{t_k=k/N_\phi\bigr\}_{k=0}^{N_\phi},
\]
and set, for $k=0,\dots,M-1$,
\[
\Delta\tau_k=\tau_{k+1}-\tau_k,\qquad
\Delta\phi_k=\phi_{x,\mathrm{PL}}(\tau_{k+1})-\phi_{x,\mathrm{PL}}(\tau_k),
\]
\[
d_k=x_\mathrm{PL}\bigl(\phi_{x,\mathrm{PL}}(\tau_k)\bigr)-y_\mathrm{PL}(\tau_k),\quad k=0,\dots,M.
\]
Then
\begin{equation}\label{eq:pl-closed-form}
\mathcal{L}(\phi_{x,\mathrm{PL}})=\frac13\sum_{k=0}^{M-1}\sqrt{\Delta\tau_k^{2}+\Delta\phi_k^{2}} \quad\times\Bigl(\lVert d_k\rVert_2^2+\langle d_k,d_{k+1}\rangle+\lVert d_{k+1}\rVert_2^2\Bigr).
\end{equation}
\end{restatable}

\paragraph{Monte-Carlo (MC) scheme.}
For any warping $\phi$, we can instead estimate $\mathcal{L}(\phi)$ by drawing $S$ i.i.d.\ samples $(t_i)_{i=1}^S\sim \mathcal{U}([0,1])$ and set
\begin{equation}\label{eq:mc-estimator}
\widehat{\mathcal{L}}^S(\phi)\;=\;\frac1S\sum_{s=1}^S \lVert x(\phi(t_s))-y(t_s)\rVert_2^2\sqrt{\dot\phi(t_s)^{2}+1}.
\end{equation}
It is well known that this estimator is unbiased and decays as $1/S$: the dimension and the sampling of $x$ and $y$ affect the multiplicative constant, but not this rate.
This estimation strategy requires nothing beyond pointwise evaluability of $\phi$, $\dot\phi$, $x$ and $y$.

\subsection{Parametrization of $\phi$}\label{sec:param}
Several options are available for the warping $\phi$, as long as they result in a differentiable and monotonic function. A natural choice could be monotonic-by-construction Neural Networks~(\textit{e.g.} \citep{Wehenkel2019UMNN,runje2023constrained}). Yet, in practice, we found that a simple piecewise-linear parametrization of $\phi$ on a uniform grid is sufficient to capture the warping between two time series, while being computationally efficient.
We hence adopt the grid parametrization of~\citet{rhodes2024timewarpvaesimultaneoustimewarpingrepresentation}, with a SoftPlus activation, which can be evaluated by both schemes: $\theta \in\mathbb R^{N_\phi}$ gives increments $\delta_j=\mathrm{SoftPlus}(\theta_j)>0$ whose normalized cumulative sums $\phi_{x,k}=\sum_{j=1}^{k}\delta_j/\sum_{j=1}^{N_\phi}\delta_j$, are the breakpoint values of a piecewise-linear $\phi$ on the uniform grid $t_k=k/N_\phi$, so that $\phi_{x,0}=0$ and $\phi_{x,N_\phi}=1$.
The grid size $N_\phi$ sets a trade-off: a finer grid lets $\phi$ represent sharper, more localized warpings, with slopes up to $N_\phi$, but increases the cost of each evaluation (linearly in $N_\phi$ for the PL scheme), while a coarser grid yields warpings that are cheaper to evaluate and easier to optimize, at the risk of missing fine alignments.
This warping belongs to $\mathcal C_{\mathrm{Lip}}$ and is strictly increasing, since $\delta_j>0$, so Propositions~\ref{prop:grid-approx} and~\ref{prop:pl-closed-form} apply.
Since $\phi_\theta$ only depends on the normalized increments $\beta_k=\delta_k/\sum_j\delta_j$, Proposition~\ref{prop:danskin} applies with $\Theta$ the closed simplex $\Delta\subset\mathbb R^{N_\phi}$, which is compact. SoftPlus only reaches the interior of $\Delta$, which leaves the minimum unchanged by density of the interior and continuity of $\beta\mapsto\mathcal L_y(\phi_\beta)$ (Lemma~\ref{lem:app-continuity}). Flat pieces of $\phi$ ($\beta_k=0$) are only approached as $\theta_k\to-\infty$, where the gradient vanishes.
In Figure~\ref{fig:alignment}, we show the alignments obtained by our FastCDTW solver for different values of $N_\phi$, the number of breakpoints in the warping function $\phi$. Even moderate values of $N_\phi$ lead to warping paths that are smoother than DTW, and distance estimation (right panel) converges quickly.

\paragraph{Symmetry.}
\llm{CDTW is symmetric: in the single-map form of Proposition~\ref{prop:reparametrization}, writing $\mathcal L_{x,y}$ for the cost~\eqref{eq:cost} of the pair $(x,y)$, the change of variables $u=\phi(t)$ gives $\mathcal L_{x,y}(\phi)=\mathcal L_{y,x}(\phi^{-1})$, and $\phi^{-1}\in\mathcal C$ whenever $\phi\in\mathcal C$, so the optimal warping for $(y,x)$ is the inverse of the one for $(x,y)$.
The grid parametrization above loses this property: the inverse of a strictly increasing piecewise-linear map on the uniform grid is piecewise linear, but with breakpoints at the non-uniform locations $\phi(t_k)$, so the family is not closed under inversion and FastCDTW is asymmetric in general.
Both directions converge to $\mathrm{CDTW}(x,y)$ as $N_\phi\to\infty$, so the asymmetry vanishes as the grid is refined.
We do not symmetrize FastCDTW in our experiments, and have not observed this asymmetry to cause any problem.}

\begin{figure}[t]
    \includegraphics[width=\columnwidth]{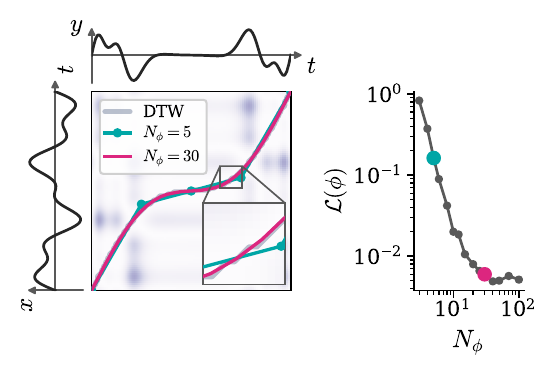}
    \caption{Alignments obtained using FastCDTW for different values of $N_\phi$, the number of breakpoints in the warping function $\phi$. The two time series only differ by a central stretching over a null value, yielding a zero DTW distance (up to sampling artifacts).\label{fig:alignment}}
\end{figure}

\subsection{Time complexities}\label{sec:complexity}

With the PL scheme, the breakpoints of Proposition~\ref{prop:pl-closed-form} are found by one sweep over those of $x_\mathrm{PL}\circ\phi_{x,\mathrm{PL}}$ and $y_\mathrm{PL}$, and each interval costs $O(d)$, so one evaluation of $\mathcal{L}(\phi)$ costs $O((N_x+N_y+N_\phi)d)$. With the MC scheme, the $N_\phi$ breakpoint values of $\phi$ are computed once, in $O(N_\phi)$, after which the uniform grid gives the interval of each sample directly by a floor operation, so $\phi$ and $\dot\phi$ are evaluated in $O(1)$. Denoting by $C_x$ and $C_y$ the costs of evaluating $x$ and $y$ at a single time, one evaluation then costs $O(N_\phi+S(d+C_x+C_y))$.
For sampled series, each sample locates its time among the samples of $x$ and $y$ by binary search, so $C_x=O(\log N_x)$ and $C_y=O(\log N_y)$ (or $O(1)$ for regularly sampled series).
When a curve is instead represented by a neural network, such as the INR barycenter of Section~\ref{subsec:averaging}, its evaluation cost is that of a forward pass through the network.
In all cases, automatic differentiation gives the gradient at the same cost (see discussion in Appendix~\ref{app:backward_cost}).
Used as a distance, FastCDTW minimizes $\mathcal{L}$ with $n_\text{iter}$ gradient steps, which multiplies these costs by $n_\text{iter}$ (Sections~\ref{subsec:averaging} and~\ref{subsec:classification}); when $\phi$ is shared across samples and trained with a model (Section~\ref{subsec:mad}), it takes one step per model update, so $n_\text{iter}=1$.
Table~\ref{tab:complexity-comparison} compares these costs with the baselines of Section~\ref{sec:background}. FastCDTW is far cheaper than the exact and approximate CDTW solvers, and than GDTW, which costs $O(N_\phi^3)$ when its grid has the resolution of our $\phi$. 
Compared with the $O(N_xN_yd)$ of DTW and Soft-DTW, with $N_x=N_y=T$ and $N_\phi\leq T$, FastCDTW is cheaper, up to constants, only when $T\gtrsim n_\text{iter}$ for the PL scheme and $T^2d\gtrsim n_\text{iter}S(d+\log T)$ for the MC scheme.
Moreover, all computations induced in FastCDTW (across the intervals for PL, across Monte Carlo samples for MC) can be efficiently parallelized on a GPU (see Figure~\ref{fig:timings} and Appendix~\ref{app:nn}).

\begin{table}[t]
\centering
\resizebox{\columnwidth}{!}{%
\begin{tabular}{lc}
\toprule
Method & Time complexity \\
\midrule
DTW~\citep{Vintsyuk1968,SakoeChiba1978} & $N_xN_yd$ \\
Soft-DTW~\citep{softdtw} & $N_xN_yd$ \\
GDTW~\citep{gdtw} & $N_\phi^3$ \\
Exact CDTW~\citep{buchin2025cdtw} & $N_x^5$ (restricted to $d=1$) \\
Approx.\ CDTW (FPTAS)~\citep{ergur2026cdtw} & $(N_x{+}N_y)^{19/6}(1/\varepsilon)^{10/3}\log\frac{N_x+N_y}{\varepsilon^2}$ \\
\midrule
FastCDTW-PL (ours) & $(N_x{+}N_y{+}N_\phi)\,d$ \\
FastCDTW-MC (ours) & $N_\phi + S(d+\log N_x+\log N_y)$ \\
\bottomrule
\end{tabular}%
}
\caption{Time complexity of computing one alignment/distance value across the different methods. FastCDTW's rows are for a single evaluation of $\mathcal{L}(\phi)$.}
\label{tab:complexity-comparison}
\end{table}

%% file: sections/experiments/averaging.tex
\paragraph{Experimental setup.}
We average $n$ series of length $T=200$ drawn from three synthetic families,
\texttt{accel} ($n=12$), \texttt{ecg} ($n=12$) and \texttt{rate} ($n=8$), each
built so that the series differ from a common template by a time
reparametrization only. 
We compare
five barycenters: the Euclidean mean; DTW Barycenter Averaging (DBA)~\citep{dba}; the Soft-DTW
barycenter~\citep{softdtw}; and two FastCDTW barycenters.
The first FastCDTW barycenter is a free vector of samples, evaluated with the PL scheme; the second is an implicit neural representation (INR)~\citep{naour2024time}, i.e.\ a function of continuous time, evaluated with the MC scheme (see Appendix~\ref{app:averaging} for details). They show that FastCDTW can handle both sampled series and continuous-time curves.
The Euclidean mean is used as initialization for all the barycenters.

\begin{figure*}[t]
 \centering
 \includegraphics[width=\textwidth]{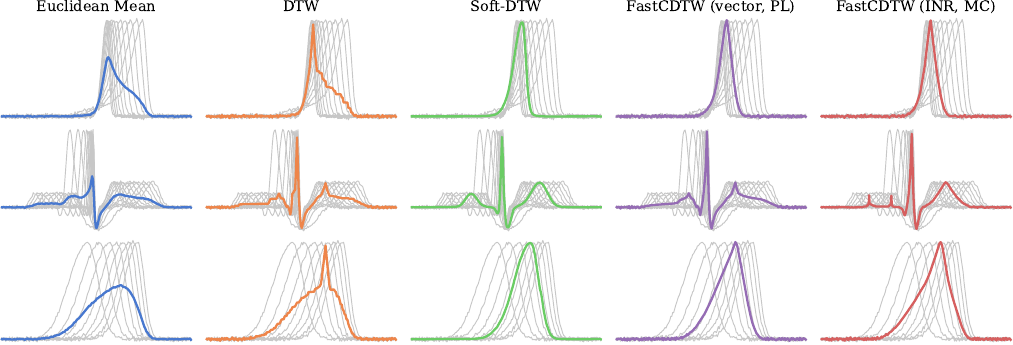}
 \caption{Barycenters of \texttt{accel} (top), \texttt{ecg} (middle) and
 \texttt{rate} (bottom). Grey: the averaged series.}
 \label{fig:barycenters-grid}
\end{figure*}

\paragraph{Quantitative results.}
Table~\ref{tab:averaging-losses} scores these barycenters under DTW, as the mean and standard deviation over the ten draws, with Soft-DTW taken, on each family, at the temperature $\gamma$ with the lowest mean DTW loss. Table~\ref{tab:averaging-gamma-sweep} and Figure~\ref{fig:gamma-sweep} give the full sweep over $\gamma$. 
On \texttt{accel} and \texttt{rate}, FastCDTW-PL
attains a lower DTW loss than DBA and than Soft-DTW, although it does not
optimize that objective. FastCDTW-MC does too on \texttt{rate}, and on
\texttt{accel} it improves on DBA but not on Soft-DTW. On \texttt{ecg}, Soft-DTW at
$\gamma=10^{-2}$ is the best of all methods.

\paragraph{Qualitative results.}
Figure~\ref{fig:barycenters-grid} shows the barycenters on the three
families, with Soft-DTW at the temperature selected above. %
Visually, both Soft-DTW and FastCDTW produce convincing barycenters that recover the shape of the common template.
On \texttt{ecg}, however, the Soft-DTW barycenter slightly erodes the central peak, a known side effect of the smoothing induced by its soft-minimum~\citep{softdtwdiv}, whereas the FastCDTW barycenters preserve it.
\llm{This erosion could be avoided by using the Soft-DTW divergence~\citep{softdtwdiv}, which corrects this bias, instead of Soft-DTW itself.}

\begin{table}[t]
 \centering
 \small
 \setlength{\tabcolsep}{3pt}
 \begin{tabular}{@{}lrrr@{}}
 \toprule
 method & \texttt{accel} & \texttt{ecg} & \texttt{rate} \\
 \midrule
 Euclidean mean & $11.88${\scriptsize$\pm$0.11} & $14.81${\scriptsize$\pm$0.06} & $13.19${\scriptsize$\pm$0.10} \\
 DTW (DBA) & $3.16${\scriptsize$\pm$0.06} & $4.43${\scriptsize$\pm$0.25} & $1.48${\scriptsize$\pm$0.08} \\
 \midrule
 Soft-DTW (best $\gamma$) & $2.31${\scriptsize$\pm$0.04} & $\mathbf{4.02}${\scriptsize$\pm$\textbf{0.03}} & $1.19${\scriptsize$\pm$0.02} \\
 \midrule
 FastCDTW-PL & $\mathbf{2.20}${\scriptsize$\pm$\textbf{0.02}} & $4.82${\scriptsize$\pm$0.43} & $\mathbf{1.08}${\scriptsize$\pm$\textbf{0.01}} \\
 FastCDTW-MC & $2.42${\scriptsize$\pm$0.05} & $4.80${\scriptsize$\pm$0.57} & $1.16${\scriptsize$\pm$0.02} \\
 \bottomrule
 \end{tabular}
 \caption{Loss $\sum_i \mathrm{DTW}(X, Y_i)$ of each barycenter (lower is better), mean $\pm$ standard deviation over ten draws of the series. Soft-DTW is taken at the temperature with the lowest mean DTW loss (Table~\ref{tab:averaging-gamma-sweep}).}
 \label{tab:averaging-losses}
\end{table}

%% file: sections/experiments/nn.tex
\paragraph{Experimental setup.}

We now evaluate the performance of the above discrepancies in terms of 1-nearest-neighbor accuracy on a subset of the UCR archive~\citep{UCRArchive}.
Of its $128$ univariate datasets, we keep $95$, each with its predefined train/test split: we exclude the $14$ with variable lengths or missing values and the ones with more than $2\times10^6$ (test, train) pairs.
For Soft-DTW, $\gamma$ is cross-validated on the training set and for the
two FastCDTW schemes, the learning rate of the warping optimization is cross-validated in the same way.

\paragraph{Results.}
Figure~\ref{fig:nn-cd} gives the average rank of each discrepancy over the
95 datasets. Soft-DTW ($2.47$), FastCDTW-PL ($2.54$) and FastCDTW-MC ($2.55$)
are not significantly different from one another, and all three are
significantly more accurate than DTW ($3.42$) and the Euclidean distance
($4.02$). Mean accuracies are $76.7\%$ for Soft-DTW, $76.3\%$ for both FastCDTW
schemes, $73.9\%$ for DTW and $70.3\%$ for the Euclidean distance, so the MC
estimate loses no accuracy compared with the exact PL integral.
Appendix~\ref{app:nn} gives the per-dataset accuracies.

\begin{figure}[t]
 \centering
 \includegraphics[width=\columnwidth]{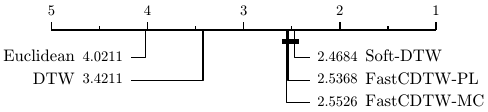}
 \caption{\ph{Average rank of each discrepancy by 1-NN test accuracy over the
 95 UCR datasets (1 is best). 
 A horizontal bar joins methods whose accuracies
 are not significantly different (pairwise two-sided Wilcoxon signed-rank tests
 with Holm's correction, $\alpha = 0.05$).}}
 \label{fig:nn-cd}
\end{figure}

%% file: sections/experiments/mad.tex
\paragraph{Experimental setup.}
Match-And-Deform~\citep{mad} (MAD) is an unsupervised domain-adaptation method for time series.
Given a labeled source domain and an unlabeled target domain, MAD jointly learns (i) a per-batch optimal-transport plan coupling source and target samples and (ii) a global temporal alignment between domains, by alternating an OT step with a DTW-based alignment step solved by dynamic programming, and then uses both the coupling and the alignment to adapt a classifier across the source and target domains.
This alternation is rerun to convergence at every loss evaluation, so the DTW step dominates the cost of the whole procedure. %

We test whether FastCDTW can replace this DTW step, removing the dynamic-programming bottleneck without lowering classification accuracy.

\paragraph{Datasets.} We reuse the two domain-adaptation benchmarks from the original MAD paper~\citep{mad}.
\emph{TarnBrittany} (TB) pairs satellite image time series of land-cover parcels from two French regions, giving two directions ($1\to2$ and $2\to1$).
\emph{miniTimeMatch} (mTM) is a reduced version of the TimeMatch crop-classification benchmark spanning three regions, yielding the four source/target pairs reported in~\citet{mad}.
For both datasets we use the same domain splits, class sets, and backbone classifier as~\citet{mad}, so that any difference in accuracy can be attributed to the choice of alignment method rather than to confounding factors in the pipeline. %

\paragraph{Results.} Table~\ref{tab:results} reports classification accuracy on the target domain, averaged over three seeds.
Substituting FastCDTW for DTW inside MAD matches or slightly improves accuracy on most source/target pairs, for both the PL and MC variants, and the two variants perform comparably to one another. 
Appendix~\ref{app:mad} further compares several parametrizations of the warping $\phi$ with the MC variant, among which the grid parametrization of Section~\ref{sec:param} yields the best downstream accuracy overall.

This substitution also changes MAD's computational profile. 
Using MAD, one postulates a global temporal shift between domains, hence in FastCDTW we train a single $\phi$ across samples, such that it can be treated as an ordinary model parameter and there is no need for a full gradient descent over $\phi$ parameters at each loss evaluation.
\llm{The bottom rows of Table~\ref{tab:results} report the resulting end-to-end speedup of MAD training over the DTW baseline, with the PL and MC variants respectively running $20\times$ and $50\times$ faster on GPU (Tesla V100).
These figures measure the whole training pipeline, and thus combine two effects: the lower cost of a single FastCDTW evaluation compared to DTW (cf Figure~\ref{fig:timings}) and the amortization of the optimization of $\phi$ across model updates described above.
Note that this amortization is not a change of training protocol that DTW could equally benefit from.
}

Overall, FastCDTW preserves the adaptation quality of MAD at a fraction of its computational cost.%

\begin{table}
  \centering
  \resizebox{\columnwidth}{!}{%
  \begin{tabular}{ccccc}
    \toprule
    Dataset & Pair & DTW & FastCDTW-MC & FastCDTW-PL \\
    \midrule
    \multirow{3}{*}{TB} & 1 $\to$ 2 & $98.6 \pm 0.3$ & $98.5 \pm 0.2$ & $\boldsymbol{98.8 \pm 0.3}$ \\
     & 2 $\to$ 1 & $89.5 \pm 0.2$ & $\boldsymbol{90.3 \pm 1.6}$ & $\boldsymbol{90.3 \pm 2.1}$ \\
    \cmidrule(lr){2-5}
     & \textit{Average} & $94.0 \pm 6.4$ & $94.4 \pm 5.8$ & $\boldsymbol{94.5 \pm 5.9}$ \\
    \midrule
    \multirow{5}{*}{mTM} & 1 $\to$ 2 & $71.2 \pm 3.5$ & $\boldsymbol{75.1 \pm 7.1}$ & $69.3 \pm 7.4$ \\
     & 1 $\to$ 3 & $72.9 \pm 15.2$ & $\boldsymbol{92.6 \pm 0.3}$ & $91.6 \pm 1.5$ \\
     & 3 $\to$ 1 & $\boldsymbol{87.3 \pm 1.1}$ & $84.0 \pm 1.4$ & $84.0 \pm 1.4$ \\
     & 3 $\to$ 2 & $84.0 \pm 1.5$ & $\boldsymbol{86.1 \pm 2.9}$ & $85.1 \pm 2.3$ \\
    \cmidrule(lr){2-5}
     & \textit{Average} & $78.9 \pm 8.0$ & $\boldsymbol{84.5 \pm 7.2}$ & $82.5 \pm 9.4$ \\
    \midrule
    \midrule
    \multirow{2}{*}{\shortstack{Speedup\\ vs.\ DTW}} & CPU & - & $3.54\times$ & $0.32\times$ \\
                                              & GPU & - & $53.93\times$ & $20.25\times$ \\
    \bottomrule
  \end{tabular}%
  }
  \caption{Domain-adaptation accuracy (mean $\pm$ std) for MAD instantiated with DTW, FastCDTW-MC, or FastCDTW-PL as the alignment step.
  }
  \label{tab:results}
\end{table}

%% file: sections/conclusion.tex
We introduced FastCDTW, a solver for Continuous Dynamic Time Warping that optimizes a parametrized warping by gradient descent. %
The alignment cost is evaluated either exactly, in closed form, or by an unbiased Monte Carlo estimator, and both schemes are differentiable with respect to the warping and to the input curves, and amenable to batched GPU computation. \nc{As such, it unlocks potential use of CDTW in machine learning pipelines, where it can be used as a loss function or as a component of a larger differentiable model.}
On the theoretical side, we showed that Lipschitz warpings reach the same optimal cost as smooth ones, and that CDTW restricted to a parametrized family of warpings is differentiable with respect to an input curve.
Empirically, FastCDTW serves as a drop-in replacement for DTW or Soft-DTW in averaging, classification and domain adaptation, at a lower computational cost. \nc{On the downside, our gradient-based solver does not guarantee global optimality because of the non-convexity of the objective. Yet, in practice, we demonstrated that it is accurate enough to be useful in machine learning pipelines.}

Because FastCDTW compares curves rather than sequences of samples, and its Monte Carlo scheme only requires evaluating them at arbitrary times, it is naturally suited to continuous-time models of time series.
A promising direction for future work is thus to use it as a training loss for continuous-time foundation models (e.g. building on implicit neural representations of time series~\citep{naour2024time}) for forecasting on irregular grids and imputation tasks. %

%% file: sections/acknowledgements.tex
This work was supported by the French National Research Agency (ANR) under the projects CALME (ANR-25-CE23-4419), OTTOPIA (ANR-20-CHIA-0030) and DATeS (ANR-25-CE23-6035), and by the ACORD chair of the SequoIA AI cluster (ANR-23-IACL-0009). CALME and OTTOPIA funded Pierre Houédry, and ACORD funded Romain Tavenard. The authors declare no competing financial interests.

%% file: sections/appendix/missingproofs.tex
\subsection{Proof of Proposition~\ref{prop:reparametrization}}

\reparametrization*

\begin{proof}
Since \(\psi\) is a strictly increasing \(C^1\)-diffeomorphism on \([0,1]\), it admits a well-defined inverse \(\psi^{-1}\) and \(\dot\psi>0\) everywhere.  Introduce the change of variables
\[
u = \psi(t),
\qquad
t = \psi^{-1}(u).
\]
Define the new reparametrization
\[
\tilde\phi(u)
:=\phi\bigl(\psi^{-1}(u)\bigr)
=(\phi\circ\psi^{-1})(u).
\]
Then
\begin{align*}
\mathcal{L}(\phi,\psi)
&=\int_{0}^{1}\bigl\lVert x(\phi(t))-y(\psi(t))\bigr\rVert_2^2
\sqrt{\dot\phi(t)^{2}+\dot\psi(t)^{2}}
\;\mathrm{d}t\\
&=\int_{0}^{1}
\Bigl\lVert x\bigl(\tilde\phi(u)\bigr)-y(u)\Bigr\rVert_2^2
\sqrt{\dot\phi\bigl(\psi^{-1}(u)\bigr)^{2}+\dot\psi\bigl(\psi^{-1}(u)\bigr)^{2}}
\;\frac{\mathrm{d}u}{\dot\psi(\psi^{-1}(u))}.
\end{align*}
Since
\[
\dot{\tilde\phi}(u)
=\frac{\mathrm{d}}{\mathrm{d}u}\,\phi\bigl(\psi^{-1}(u)\bigr)
=\frac{\dot\phi(\psi^{-1}(u))}{\dot\psi(\psi^{-1}(u))},
\]
we have
\[
\sqrt{\dot\phi(\psi^{-1}(u))^{2}+\dot\psi(\psi^{-1}(u))^{2}}
\;\frac{1}{\dot\psi(\psi^{-1}(u))}
=\sqrt{\dot{\tilde\phi}(u)^{2}+1}.
\]
Combining these identities yields
\[
\mathcal{L}(\phi,\psi)
=\int_{0}^{1}\bigl\lVert x(\tilde\phi(u))-y(u)\bigr\rVert_2^2
\sqrt{\dot{\tilde\phi}(u)^{2}+1}
\;\mathrm{d}u
=\mathcal{L}\bigl(\tilde\phi,\mathrm{Id}\bigr),
\]
as required.
\end{proof}

\subsection{Proof of Proposition~\ref{prop:grid-approx}}\label{app:grid-approx}

\gridapprox*

Recall that $\mathcal C$ is the set of increasing $C^1$-diffeomorphisms $\phi:[0,1]\to[0,1]$ with $\phi(0)=0$ and $\phi(1)=1$ (Section~\ref{sec:background}).
Throughout, $R:=\max_{u}\lVert x(u)\rVert_2+\max_{t}\lVert y(t)\rVert_2$ which is finite since $x$ and $y$ are continuous, and $\omega_x(s):=\sup_{\lvert u-v\rvert\leq s}\lVert x(u)-x(v)\rVert_2$ is the modulus of continuity of $x$: it is nondecreasing, and $\omega_x(s)\to0$ as $s\to0$ since $x$ is uniformly continuous on $[0,1]$. Every $\phi\in\mathcal C_{\mathrm{Lip}}$ is differentiable almost everywhere with $\dot\phi\in L^\infty([0,1])$, so $\mathcal L(\phi)$ is well defined on $\mathcal C_{\mathrm{Lip}}$, and $\mathcal C\subset\mathcal C_{\mathrm{Lip}}$.

\begin{lemma}[continuity of $\mathcal L$]\label{lem:app-L-continuity}
Let $\phi,\phi_n\in\mathcal C_{\mathrm{Lip}}$ with $\sup_n\lVert\dot\phi_n\rVert_\infty\leq M$ and $\varepsilon_n:=\lVert\dot\phi_n-\dot\phi\rVert_{L^1}\to0$. Then $\mathcal L(\phi_n)\to\mathcal L(\phi)$.
\end{lemma}

\begin{proof}
Since $\phi_n(0)=\phi(0)=0$, $\lvert\phi_n(t)-\phi(t)\rvert=\bigl\lvert\int_0^t\bigl(\dot\phi_n(s)-\dot\phi(s)\bigr)\,\mathrm ds\bigr\rvert\leq\varepsilon_n$. Let $r_n:=x\circ\phi_n-y$, $r:=x\circ\phi-y$, $w_n:=\sqrt{\dot\phi_n^2+1}$ and $w:=\sqrt{\dot\phi^2+1}$, so that
\[
\mathcal L(\phi_n)-\mathcal L(\phi)=\int_0^1\bigl(\lVert r_n\rVert_2^2-\lVert r\rVert_2^2\bigr)w_n\,\mathrm dt+\int_0^1\lVert r\rVert_2^2\,(w_n-w)\,\mathrm dt.
\]
Recall that for $a,b\in\mathbb R^d$, $\lVert a\rVert_2^2-\lVert b\rVert_2^2=\langle a-b,a+b\rangle$. By the Cauchy-Schwarz and triangle inequalities, we have
\[
\bigl\lvert\lVert a\rVert_2^2-\lVert b\rVert_2^2\bigr\rvert\leq\lVert a-b\rVert_2\bigl(\lVert a\rVert_2+\lVert b\rVert_2\bigr).
\]
We apply this with $a=r_n(t)$ and $b=r(t)$. First, $\lVert r_n(t)\rVert_2\leq\lVert x(\phi_n(t))\rVert_2+\lVert y(t)\rVert_2\leq R$, and likewise $\lVert r(t)\rVert_2\leq R$. Second, the $y(t)$ terms cancel, so $r_n(t)-r(t)=x(\phi_n(t))-x(\phi(t))$. Since $\lvert\phi_n(t)-\phi(t)\rvert\leq\varepsilon_n$ and $\omega_x$ is nondecreasing, $\lVert r_n(t)-r(t)\rVert_2\leq\omega_x(\varepsilon_n)$. Hence, for every $t$,
\[
\bigl\lvert\lVert r_n(t)\rVert_2^2-\lVert r(t)\rVert_2^2\bigr\rvert\leq2R\,\omega_x(\varepsilon_n).
\]
Moreover $w_n\leq\sqrt{M^2+1}$, and $\lvert w_n-w\rvert\leq\lvert\dot\phi_n-\dot\phi\rvert$ since $a\mapsto\sqrt{a^2+1}$ is $1$-Lipschitz. Hence
\[
\lvert\mathcal L(\phi_n)-\mathcal L(\phi)\rvert\leq2R\sqrt{M^2+1}\;\omega_x(\varepsilon_n)+R^2\varepsilon_n\longrightarrow0 .\qedhere
\]
\end{proof}

\begin{proof}[Proof of Proposition~\ref{prop:grid-approx}]
Since $\mathcal C\subset\mathcal C_{\mathrm{Lip}}$, $\inf_{\mathcal C_{\mathrm{Lip}}}\mathcal L\leq\inf_{\mathcal C}\mathcal L$. Conversely, let $\phi\in\mathcal C_{\mathrm{Lip}}$ and $M:=\lVert\dot\phi\rVert_\infty$. Then $\dot\phi\in[0,M]$ almost everywhere and $\int_0^1\dot\phi(t)\,\mathrm dt=\phi(1)-\phi(0)=1$. Continuous functions are dense in $L^1([0,1])$, and the projection $a\mapsto\min(\max(a,0),M)$ onto $[0,M]$ is $1$-Lipschitz and leaves $\dot\phi$ unchanged, so there are continuous $g_n$ with values in $[0,M]$ such that $\lVert g_n-\dot\phi\rVert_{L^1}\to0$. In particular $c_n:=\int_0^1g_n(t)\,\mathrm dt\to1$. For $n$ large enough that $c_n>1/2$, set
\[
\tilde g_n:=\Bigl(1-\frac1n\Bigr)\frac{g_n}{c_n}+\frac1n,\qquad\phi_n(t):=\int_0^t\tilde g_n(s)\,\mathrm ds .
\]
Then $\tilde g_n$ is continuous, $\tilde g_n\geq1/n>0$ and $\int_0^1\tilde g_n(t)\,\mathrm dt=1$, so $\phi_n$ is $C^1$ with positive derivative, $\phi_n(0)=0$ and $\phi_n(1)=1$, that is $\phi_n\in\mathcal C$. Hence,
\[
\lVert\tilde g_n-\dot\phi\rVert_{L^1}\leq\Bigl\lVert\frac{g_n}{c_n}-\dot\phi\Bigr\rVert_{L^1}+\frac2n\leq\frac{\lVert g_n-\dot\phi\rVert_{L^1}}{c_n}+\Bigl\lvert\frac1{c_n}-1\Bigr\rvert+\frac2n\longrightarrow0 .
\]
By Lemma~\ref{lem:app-L-continuity}, $\mathcal L(\phi_n)\to\mathcal L(\phi)$, so $\inf_{\mathcal C}\mathcal L\leq\mathcal L(\phi)$. Since $\phi\in\mathcal C_{\mathrm{Lip}}$ is arbitrary, $\inf_{\mathcal C}\mathcal L\leq\inf_{\mathcal C_{\mathrm{Lip}}}\mathcal L$.
\end{proof}

\subsection{Proof of Proposition~\ref{prop:danskin}}\label{app:danskin}

\danskin*

For $\theta\in\Theta$ and $y\in\mathcal Y$ we write
\[
w_\theta:=\sqrt{\dot\phi_\theta^{\,2}+1},\quad
r_{\theta,y}:=x\circ\phi_\theta-y,\quad
V(y):=\mathrm{CDTW}_\Theta(x,y)=\inf_{\theta\in\Theta}\mathcal L_y(\phi_\theta),
\]
and $\lVert x\rVert_\infty:=\max_{u\in[0,1]}\lVert x(u)\rVert_2$, finite by (A1). By (A2) and (A4), $\theta\mapsto\lVert\dot\phi_\theta\rVert_\infty$ is continuous on a compact set, so $B:=\max_{\theta\in\Theta}\lVert\dot\phi_\theta\rVert_\infty<\infty$ and we set $\bar w:=\sqrt{B^2+1}$. This uniform slope bound is what allows $y$ and $h$ to be taken in $L^2$. Since $\phi_\theta$ takes values in $[0,1]$,
\begin{equation}\label{eq:app-w-bound}
1\leq w_\theta\leq\bar w\ \text{ a.e.},\qquad \lVert r_{\theta,y}\rVert_{L^2}\leq\lVert x\rVert_\infty+\lVert y\rVert_{L^2} .
\end{equation}

\begin{lemma}[$\mathcal L_y(\phi)$ is quadratic in $y$]\label{lem:app-quadratic}
Let $\phi:[0,1]\to[0,1]$ be Lipschitz, $w:=\sqrt{\dot\phi^2+1}$ and $r:=x\circ\phi-y$. For every $y,h\in\mathcal Y$ and $\tau\in\mathbb R$,
\[
\mathcal L_{y+\tau h}(\phi)=\mathcal L_y(\phi)+\tau\,\partial_h\mathcal L_y(\phi)+\tau^2\int_0^1w\lVert h\rVert_2^2\,\mathrm dt ,
\]
with $\partial_h\mathcal L_y(\phi)$ given by \eqref{eq:dir-deriv-explicit}, and
\[
\bigl\lvert\partial_h\mathcal L_y(\phi)\bigr\rvert\leq2\lVert w\rVert_\infty\lVert r\rVert_{L^2}\,\lVert h\rVert_{L^2},\qquad
0\le\int_0^1w\lVert h\rVert_2^2\,\mathrm dt\leq\lVert w\rVert_\infty\lVert h\rVert_{L^2}^2 .
\]
In particular the limit defining $\partial_h\mathcal L_y(\phi)$ exists and is two-sided, and $h\mapsto\partial_h\mathcal L_y(\phi)=\langle-2w\,r,h\rangle_{L^2}$ is linear and continuous on $\mathcal Y$.
\end{lemma}

\begin{proof}
All integrals are finite, since $w$ is bounded, $x\circ\phi$ is continuous, and $y,h\in\mathcal Y$. Expand $\lVert r-\tau h\rVert_2^2=\lVert r\rVert_2^2-2\tau\langle r,h\rangle+\tau^2\lVert h\rVert_2^2$ pointwise, multiply by $w$ and integrate. The bounds follow from $0\le w\le\lVert w\rVert_\infty$ and the Cauchy-Schwarz inequality in $L^2$.
\end{proof}

\begin{lemma}[continuity in $\theta$]\label{lem:app-continuity}
Under (A1)-(A4), $(\theta,y)\mapsto\mathcal L_y(\phi_\theta)$ is continuous on $\Theta\times\mathcal Y$, and for fixed $y,h\in\mathcal Y$ the map $\theta\mapsto\partial_h\mathcal L_y(\phi_\theta)$ is continuous on $\Theta$.
\end{lemma}

\begin{proof}
We follow the proof of Lemma~\ref{lem:app-L-continuity}, with two changes. First, $y$ is only in $L^2$, so the pointwise bound by $R$ is replaced by the Cauchy-Schwarz inequality in $L^2$ together with \eqref{eq:app-w-bound}. Second, $\varepsilon:=\lVert\dot\phi_\theta-\dot\phi_{\theta'}\rVert_\infty$ is a sup-norm distance, which tends to $0$ by (A4), so that $\lVert\phi_\theta-\phi_{\theta'}\rVert_\infty\leq\varepsilon$ and $\lvert w_\theta-w_{\theta'}\rvert\leq\varepsilon$ almost everywhere. This gives $\lVert r_{\theta,y}-r_{\theta',y'}\rVert_{L^2}\leq\omega_x(\varepsilon)+\lVert y-y'\rVert_{L^2}$ and
\[
\lvert \mathcal L_y(\phi_\theta)-\mathcal L_{y'}(\phi_{\theta'})\rvert\leq\bar w\bigl(\lVert r_{\theta,y}\rVert_{L^2}+\lVert r_{\theta',y'}\rVert_{L^2}\bigr)\bigl(\omega_x(\varepsilon)+\lVert y-y'\rVert_{L^2}\bigr)+\varepsilon\lVert r_{\theta',y'}\rVert_{L^2}^2 ,
\]
which tends to $0$ as $(\theta',y')\to(\theta,y)$, since $\lVert r_{\theta',y'}\rVert_{L^2}\le\lVert x\rVert_\infty+\lVert y'\rVert_{L^2}$ stays bounded. The same splitting applied to \eqref{eq:dir-deriv-explicit} gives, for fixed $y$ and $h$,
\[
\bigl\lvert\partial_h\mathcal L_y(\phi_\theta)-\partial_h\mathcal L_y(\phi_{\theta'})\bigr\rvert\leq2\lVert h\rVert_{L^2}\bigl(\bar w\,\omega_x(\varepsilon)+\varepsilon\lVert r_{\theta',y}\rVert_{L^2}\bigr)\longrightarrow0 .
\]
\end{proof}

\begin{proof}[Proof of Proposition~\ref{prop:danskin}]
\emph{(i) Existence and compactness.} By Lemma~\ref{lem:app-continuity}, $\theta\mapsto\mathcal L_y(\phi_\theta)$ is continuous on $\Theta$, which is compact by (A2). By the extreme value theorem, it attains its infimum: there is $\theta^\star\in\Theta$ with $\mathcal L_y(\phi_{\theta^\star})=V(y)$, so $\Phi(y)\neq\emptyset$ and the infimum defining $\mathrm{CDTW}_\Theta(x,y)$ is a minimum. Moreover, $\Phi(y)=\{\theta\in\Theta:\mathcal L_y(\phi_\theta)=V(y)\}$ is closed in $\Theta$ as the preimage of the closed set $\{V(y)\}$ under a continuous map, hence compact.

\emph{(ii) Upper bound.} Let $\theta^\star\in\Phi(y)$ and $\tau>0$. By Lemma~\ref{lem:app-quadratic},
\[
V(y+\tau h)\leq \mathcal L_{y+\tau h}(\phi_{\theta^\star})=V(y)+\tau\,\partial_h\mathcal L_y(\phi_{\theta^\star})+\tau^2\int_0^1w_{\theta^\star}\lVert h\rVert_2^2\,\mathrm dt ,
\]
and the last term is at most $\tau^2\bar w\lVert h\rVert_{L^2}^2$, so $\limsup_{\tau\to0^+}\bigl(V(y+\tau h)-V(y)\bigr)/\tau\leq\partial_h\mathcal L_y(\phi_{\theta^\star})$. The right-hand side is continuous in $\theta^\star$ (Lemma~\ref{lem:app-continuity}) on the compact set $\Phi(y)$, so its minimum over $\Phi(y)$ is attained, and
\[
\limsup_{\tau\to0^+}\frac{V(y+\tau h)-V(y)}{\tau}\leq\min_{\theta\in\Phi(y)}\partial_h\mathcal L_y(\phi_\theta).
\]

\emph{(ii) Lower bound.} Let $\tau_n\to0^+$ be a sequence along which $\bigl(V(y+\tau_nh)-V(y)\bigr)/\tau_n$ converges to the $\liminf$, in $[-\infty,+\infty]$. Pick $\theta_n\in\Phi(y+\tau_nh)$, which is possible by (i). By Lemma~\ref{lem:app-quadratic}, whose quadratic term is nonnegative,
\[
V(y+\tau_nh)=\mathcal L_{y+\tau_nh}(\phi_{\theta_n})\geq \mathcal L_y(\phi_{\theta_n})+\tau_n\,\partial_h\mathcal L_y(\phi_{\theta_n})\geq V(y)+\tau_n\,\partial_h\mathcal L_y(\phi_{\theta_n}).
\]
By compactness of $\Theta$, a subsequence of $(\theta_n)$ converges to some $\bar\theta\in\Theta$. We re-index $(\tau_n)$ and $(\theta_n)$ along it, which does not change the limit of the difference quotients, so that $\theta_n\to\bar\theta$.
For $\theta^\star\in\Phi(y)$, $\mathcal L_{y+\tau_nh}(\phi_{\theta_n})=V(y+\tau_nh)\leq \mathcal L_{y+\tau_nh}(\phi_{\theta^\star})$, so by Lemma~\ref{lem:app-continuity}, $\mathcal L_y(\phi_{\bar\theta})=\lim_n\mathcal L_{y+\tau_nh}(\phi_{\theta_n})\leq \mathcal L_y(\phi_{\theta^\star})=V(y)$, that is $\bar\theta\in\Phi(y)$. Dividing by $\tau_n$ and letting $n\to\infty$, Lemma~\ref{lem:app-continuity} gives
\[
\liminf_{\tau\to0^+}\frac{V(y+\tau h)-V(y)}{\tau}\geq\lim_n\partial_h\mathcal L_y(\phi_{\theta_n})=\partial_h\mathcal L_y(\phi_{\bar\theta})\geq\min_{\theta\in\Phi(y)}\partial_h\mathcal L_y(\phi_\theta).
\]
The combination of the two bounds gives \eqref{eq:fmtw-gradient-y}.
\end{proof}

\subsection{Proof of Proposition~\ref{prop:pl-closed-form}}\label{app:pl-closed-form}

\plclosedform*

\begin{proof}
\emph{Step 1: the sum decomposition.} The map $y_\mathrm{PL}$ changes affine pieces only at its own breakpoints $\{t_{y,j}\}$, and $\phi_{x,\mathrm{PL}}$ only at the grid points $\{k/N_\phi\}$. Consequently $x_\mathrm{PL}\circ\phi_{x,\mathrm{PL}}$ changes affine pieces only where $\phi_{x,\mathrm{PL}}$ does, or where $\phi_{x,\mathrm{PL}}(t)$ crosses a breakpoint of $x_\mathrm{PL}$, that is at the preimages $\{\phi_{x,\mathrm{PL}}^{-1}(t_{x,i})\}$. All three sets belong to $\{\tau_k\}_{k=0}^M$ by assumption, so the three maps are affine on every interval $[\tau_k,\tau_{k+1}]$, which gives the decomposition
\begin{equation}\label{eq:app-split}
\mathcal{L}(\phi_{x,\mathrm{PL}})=\sum_{k=0}^{M-1}\int_{\tau_k}^{\tau_{k+1}}
\bigl\lVert d(t)\bigr\rVert_2^2\,
\sqrt{\dot\phi_{x,\mathrm{PL}}(t)^{2}+1}\;\mathrm dt,
\end{equation}
where $d(t)=x_\mathrm{PL}(\phi_{x,\mathrm{PL}}(t))-y_\mathrm{PL}(t)$ denotes the residual.

\emph{Step 2: one interval.} Fix $k$ and work on $[\tau_k,\tau_{k+1}]$. There $\dot\phi_{x,\mathrm{PL}}$ is the constant $\Delta\phi_k/\Delta\tau_k$, so the speed factor is constant and equal to
\[
\sqrt{\dot\phi_{x,\mathrm{PL}}^{2}+1}
=\frac{\sqrt{\Delta\tau_k^{2}+\Delta\phi_k^{2}}}{\Delta\tau_k}.
\]
Moreover $d$ is a difference of two affine maps of $t$, hence affine, hence equal to the linear interpolation of its endpoint values $d_k$ and $d_{k+1}$. The $k$-th term of \eqref{eq:app-split} is therefore
\[
\frac{\sqrt{\Delta\tau_k^{2}+\Delta\phi_k^{2}}}{\Delta\tau_k}
\int_{\tau_k}^{\tau_{k+1}}\bigl\lVert d_k+\tfrac{t-\tau_k}{\Delta\tau_k}(d_{k+1}-d_k)\bigr\rVert_2^2\,\mathrm dt.
\]

\emph{Step 3: integral.} The change of variables $s=(t-\tau_k)/\Delta\tau_k$ turns it into $\Delta\tau_k\int_0^1\lVert d_k+s(d_{k+1}-d_k)\rVert_2^2\,\mathrm ds$, and expanding the squared norm gives, for any two vectors $a,b\in\mathbb R^d$,
\[
\int_0^1\bigl\lVert a+s(b-a)\bigr\rVert_2^2\,\mathrm ds
=\frac{\lVert a\rVert_2^2+\langle a,b\rangle+\lVert b\rVert_2^2}{3},
\]
since $\int_0^1(a_i+s(b_i-a_i))^2\,\mathrm ds=(a_i^2+a_ib_i+b_i^2)/3$ in each coordinate. The $k$-th term of \eqref{eq:app-split} is thus
\[
\frac{\sqrt{\Delta\tau_k^{2}+\Delta\phi_k^{2}}}{3}
\Bigl(\lVert d_k\rVert_2^2+\langle d_k,d_{k+1}\rangle+\lVert d_{k+1}\rVert_2^2\Bigr),
\]
the factors $\Delta\tau_k$ cancelling. Summing over $k$ gives \eqref{eq:pl-closed-form}.
\end{proof}
If $\phi_{x,\mathrm{PL}}$ has flat pieces, the proposition still holds with $\phi_{x,\mathrm{PL}}^{-1}(t_{x,i})$ taken as any point of its preimage, since the endpoints of a flat piece are grid points, which already belong to $\{\tau_k\}$.

%% file: sections/appendix/complexity.tex
\llm{The complexity of the backward pass of FastCDTW is the same as that of the forward pass, up to a constant factor, for both the PL and MC schemes.
For the MC scheme, this is straightforward: the Monte Carlo samples are drawn independently of $\theta$, so the loss is an average of $S$ terms, each a closed-form differentiable function of $\theta$ (including $\dot{\phi}_x(\tau_s)$, obtained by automatic differentiation).
For the PL scheme, the backward pass is more subtle because the breakpoints $\{\tau_k\}$ of the piecewise-affine integrand in Proposition~\ref{prop:pl-closed-form} depend on $\theta$.}

\paragraph{Forward pass.}
\llm{The breakpoints $\{\tau_k\}$ of Proposition~\ref{prop:pl-closed-form} come from three sorted sets.
The grid points $t_k=k/N_\phi$ and the sampling times $t_{y,j}$ do not depend on $\theta$.
The preimages $\phi_{x,\mathrm{PL}}^{-1}(t_{x,i})$ do: each one is obtained by locating the grid interval on which $\phi_{x,\mathrm{PL}}$ reaches $t_{x,i}$, then inverting the affine piece of $\phi_{x,\mathrm{PL}}$ on this interval.
Since $\phi_{x,\mathrm{PL}}$ is increasing and the $t_{x,i}$ are sorted, these preimages are already sorted, and the three sets are merged into $\{\tau_k\}$ in $O(N_x+N_y+N_\phi)$ (our implementation simply uses a generic sort, which only adds a logarithmic factor).
Equation~\eqref{eq:pl-closed-form} is then evaluated in $O(d)$ per interval.}

\paragraph{Backward pass.}
\llm{The operations above are of two kinds.
Locating intervals and merging the three sets only output integer indices and a permutation: for almost every $\theta$ (i.e.\ away from the values at which two breakpoints coincide), these are locally constant in $\theta$, so automatic differentiation treats them as constants, and the backward pass of the merge reduces to applying the same permutation to the incoming gradients.
All other operations (the SoftPlus and normalized cumulative sum defining $\phi_{x,\mathrm{PL}}$, the affine inversions, the evaluation of $x_\mathrm{PL}$ and $y_\mathrm{PL}$ at the $\tau_k$, and the sum~\eqref{eq:pl-closed-form}) are differentiable closed-form expressions.
Hence no root-finding is differentiated through and no custom backward pass is needed: by the usual property of reverse-mode automatic differentiation, the gradient costs a constant factor times the forward pass. %
}

%% file: sections/appendix/averaging.tex
\subsection{Time series averaging}\label{app:averaging}

\paragraph{Hardware.}
No GPU is used in this section: every run is on CPU. The runs were spread over several machines, the fastest per core being an Intel Xeon Gold 6254 CPU at 3.10~GHz.

\paragraph{Toy datasets description.}
Each family has a template $z$ and one increasing warping
$w_i \colon [0,1] \to [0,1]$ per series, and series $i$ is generated as
\begin{equation*}
    Y_i(t_k) = z\bigl(w_i(t_k)\bigr) + \varepsilon_{ik},
    \text{ where } t_k = \tfrac{k}{T-1} \text{ for } k = 0,\dots,T-1, \; T = 200, \text{ and } \varepsilon_{ik}\sim\mathcal{N}(0,\sigma^2) \text{ with } \sigma = 0.01.
\end{equation*}
Writing
\begin{equation*}
    g_{\mu,s,h}(t) = h\exp\Bigl(-\tfrac{(t-\mu)^2}{2s^2}\Bigr),
    \quad
    p_a(x) = x^{a},
    \quad
    q_a(x) = x + \tfrac{a}{2\pi}\sin(2\pi x),
\end{equation*}
the three families differ only in $z$ and in the warpings $w_i$:

\begin{itemize}
    \item \texttt{rate} ($n=8$): $z = g_{0.60,\,0.090,\,1}$ on $t\le 0.60$ and
    $g_{0.60,\,0.045,\,1}$ beyond, and $w_i = p_{a_i}$ with $\lbrace a_i \rbrace_{i=1}^{8}$
    evenly spaced in $[0.65, 1.55]$, so each series stays ahead of $z$ ($a_i<1$) or behind it ($a_i>1$) on the whole interval.

    \item \texttt{accel} ($n=12$): $z = g_{0.60,\,0.045,\,1}$ on $t\le 0.60$ and
    $g_{0.60,\,0.022,\,1}$ beyond (the \texttt{rate} peak with both widths halved), and
    $w_i = q_{a_i}$ with $\lbrace a_i \rbrace_{i=1}^{12}$ evenly spaced in $[-0.9, 0.9]$,
    so each series runs alternately faster and slower than $z$.

    \item \texttt{ecg} ($n=12$): $w_i = q_{a_i}$ as in \texttt{accel}, and
    \begin{equation*}
        z = g_{0.30,\,0.032,\,0.18} - g_{0.445,\,0.008,\,0.12}
          + g_{0.47,\,0.009,\,1.00} - g_{0.50,\,0.011,\,0.28}
          + g_{0.66,\,0.048,\,0.32}.
    \end{equation*}
\end{itemize}

The parameters $a_i$ of the warpings are fixed in each family, so two draws of the series differ only through the noise $\varepsilon_{ik}$. Tables~\ref{tab:averaging-losses} and~\ref{tab:averaging-gamma-sweep} report the mean and standard deviation over ten draws.
Figures~\ref{fig:barycenters-grid} and~\ref{fig:gamma-sweep} show the draw of one seed.

\paragraph{INR barycenter.}\label{app:inr}
An implicit neural representation (INR)~\citep{naour2024time} represents the barycenter as a function $c_\omega:[0,1]\to\mathbb{R}^d$, given by a neural network with parameters $\omega$, instead of a vector of values at fixed timestamps. A time $t$ is first mapped to $2L$ Fourier features
\begin{equation*}
    \gamma(t)=\bigl(\sin(2\pi f_1 t),\cos(2\pi f_1 t),\dots,\sin(2\pi f_L t),\cos(2\pi f_L t)\bigr),
    \quad\text{with } f_l = 2^{l-2} \text{ for } l=1,\dots,L,
\end{equation*}
so that the frequencies are spaced geometrically, each twice the previous one, from $f_1=1/2$ to $f_L=2^{L-2}$. These features let the network represent sharp peaks. An MLP with ReLU activations then maps $\gamma(t)$ to $c_\omega(t)$ (Figure~\ref{fig:inr}). In every experiment of this section, $L=10$, so the frequencies range from $1/2$ to $256$, and the MLP has $5$ hidden layers of width $64$. The warpings $\phi_i$ and the parameters $\omega$ are optimized by gradient descent on the FastCDTW loss. Since $c_\omega$ can be evaluated at any time, it is used with the MC scheme, and after training it can be read at any resolution, independently of the sampling grid of the averaged series (Figure~\ref{fig:inr-continuous}).

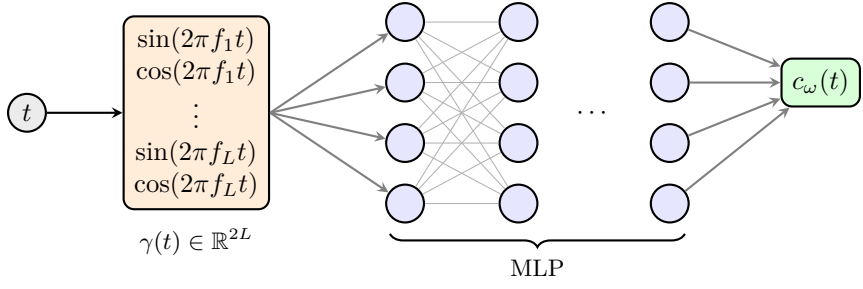
\begin{figure}[t]
\centering
\begin{tikzpicture}[
  >=stealth, thick,
  neuron/.style={circle, draw, minimum size=5mm, inner sep=0pt},
  box/.style={draw, rounded corners, align=center, inner sep=4pt},
]
  \node[neuron, fill=gray!15] (t) at (0,0) {$t$};
  \node[box, fill=orange!15, right=1.0cm of t] (ff)
    {$\sin(2\pi f_1 t)$\\ $\cos(2\pi f_1 t)$\\ $\vdots$\\ $\sin(2\pi f_L t)$\\ $\cos(2\pi f_L t)$};
  \node[below=1mm of ff, font=\small] {$\gamma(t)\in\mathbb{R}^{2L}$};
  \foreach \l/\x in {1/5.0, 2/6.5, 3/8.5} {
    \foreach \n/\y in {1/1.2, 2/0.4, 3/-0.4, 4/-1.2} {
      \node[neuron, fill=blue!10] (h\l\n) at (\x,\y) {};
    }
  }
  \node at (7.5,0) {$\cdots$};
  \foreach \n in {1,2,3,4} { \draw[->, gray] (ff.east) -- (h1\n); }
  \foreach \a in {1,2,3,4} { \foreach \b in {1,2,3,4} { \draw[gray!60, thin] (h1\a) -- (h2\b); } }
  \draw[decorate, decoration={brace, amplitude=5pt, mirror}] (4.8,-1.6) -- (8.7,-1.6)
    node[midway, below=6pt, font=\small] {MLP};
  \node[box, fill=green!15, right=1.2cm of h32] (out) {$c_\omega(t)$};
  \foreach \n in {1,2,3,4} { \draw[->, gray] (h3\n) -- (out); }
  \draw[->] (t) -- (ff);
\end{tikzpicture}
\caption{INR barycenter: a time $t\in[0,1]$ is mapped to Fourier features $\gamma(t)$, then an MLP returns the value $c_\omega(t)\in\mathbb{R}^d$ of the barycenter at time $t$. In our experiments, $L=10$ and the MLP has $5$ hidden layers of width $64$.}
\label{fig:inr}
\end{figure}

\begin{table}[tp]
 \centering
 \small
 \setlength{\tabcolsep}{4pt}
 \begin{tabular}{lrrr}
 \toprule
 method & \texttt{accel} & \texttt{ecg} & \texttt{rate} \\
 \midrule
 Euclidean mean & $11.88${\scriptsize$\pm$0.11} & $14.81${\scriptsize$\pm$0.06} & $13.19${\scriptsize$\pm$0.10} \\
 DTW (DBA) & $3.16${\scriptsize$\pm$0.06} & $4.43${\scriptsize$\pm$0.25} & $1.48${\scriptsize$\pm$0.08} \\
 \midrule
 Soft-DTW $\gamma=10$ & $15.92${\scriptsize$\pm$0.04} & $22.11${\scriptsize$\pm$0.04} & $9.75${\scriptsize$\pm$0.03} \\
 Soft-DTW $\gamma=1$ & $2.72${\scriptsize$\pm$0.05} & $10.34${\scriptsize$\pm$0.06} & $1.32${\scriptsize$\pm$0.02} \\
 Soft-DTW $\gamma=10^{-1}$ & $2.40${\scriptsize$\pm$0.07} & $4.77${\scriptsize$\pm$0.05} & $1.22${\scriptsize$\pm$0.01} \\
 Soft-DTW $\gamma=10^{-2}$ & $2.31${\scriptsize$\pm$0.04} & $\mathbf{4.02}${\scriptsize$\pm$\textbf{0.03}} & $1.22${\scriptsize$\pm$0.01} \\
 Soft-DTW $\gamma=10^{-3}$ & $3.02${\scriptsize$\pm$0.09} & $4.25${\scriptsize$\pm$0.43} & $1.19${\scriptsize$\pm$0.02} \\
 Soft-DTW $\gamma=10^{-4}$ & $3.10${\scriptsize$\pm$0.09} & $4.37${\scriptsize$\pm$0.19} & $1.33${\scriptsize$\pm$0.05} \\
 \midrule
 FastCDTW-PL & $\mathbf{2.20}${\scriptsize$\pm$\textbf{0.02}} & $4.82${\scriptsize$\pm$0.43} & $\mathbf{1.08}${\scriptsize$\pm$\textbf{0.01}} \\
 FastCDTW-MC & $2.42${\scriptsize$\pm$0.05} & $4.80${\scriptsize$\pm$0.57} & $1.16${\scriptsize$\pm$0.02} \\
 \bottomrule
 \end{tabular}
 \caption{DTW loss of each barycenter, as in Table~\ref{tab:averaging-losses},
 with Soft-DTW at every temperature $\gamma$ of the sweep. Mean $\pm$ standard deviation over the same ten draws of the series.}
 \label{tab:averaging-gamma-sweep}
\end{table}

\paragraph{Hyperparameters.}
Series are used without $z$-normalization. Every barycenter starts from the Euclidean mean, and every warping $\phi_i$ from the identity ($\theta=0$). DBA and the Soft-DTW barycenter use the implementations of \texttt{tslearn}~\citep{tslearn}. Soft-DTW is run for $\gamma\in\{10^{-4},10^{-3},10^{-2},10^{-1},1,10\}$, and the main text reports, on each family, the value of $\gamma$ with the lowest DTW loss (Table~\ref{tab:averaging-gamma-sweep}). Both FastCDTW solvers usex the warping parametrization of Section~\ref{sec:param} with $N_\phi=T-1$, and alternate $30$ times between a step on the warpings and a step on the centroid, each step being $500$ Adam iterations. The learning rate is $10^{-3}$ for the vector solver. The INR centroid is the network of Figure~\ref{fig:inr}; before the alternation, it is fitted to the Euclidean mean by $2000$ Adam steps with learning rate $10^{-2}$. The learning rate of the INR solver is then $5\times10^{-3}$ on \texttt{accel} and \texttt{rate} and $1.5\times10^{-2}$ on \texttt{ecg}, and its MC scheme draws $S=200$ samples per step on \texttt{accel} and \texttt{rate} and $S=800$ on \texttt{ecg}.
We take $L=10$ for every INR run of this section.

\begin{figure}[tp]
 \centering
 \includegraphics[width=\columnwidth]{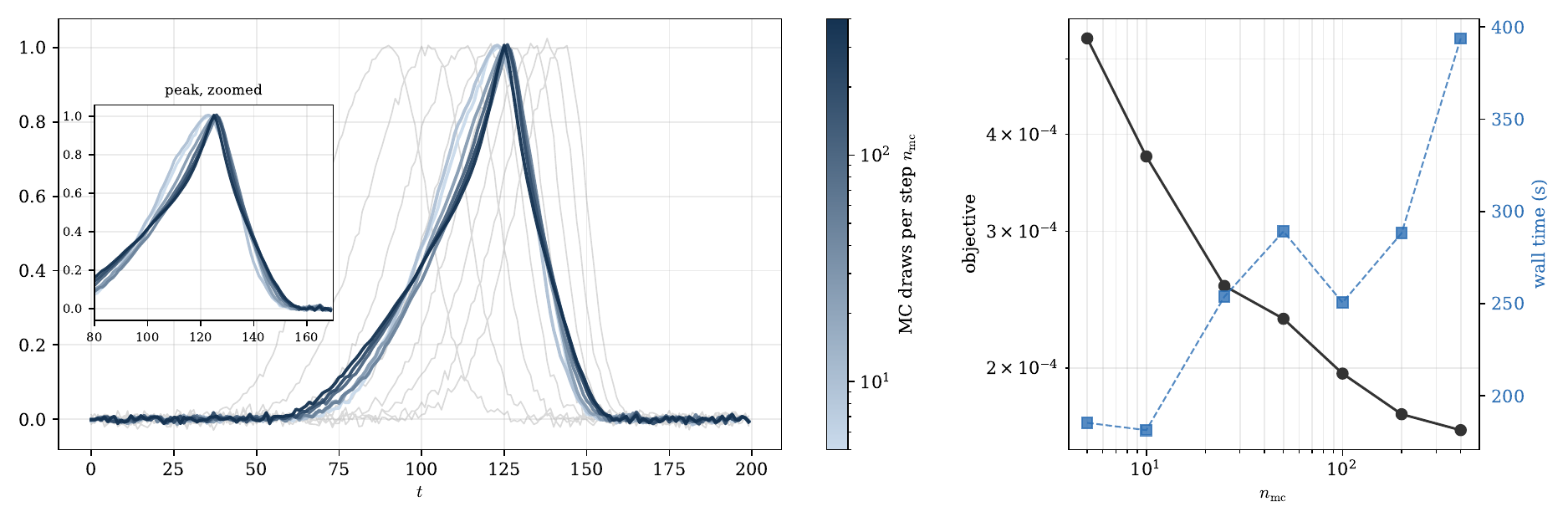}
 \includegraphics[width=\columnwidth]{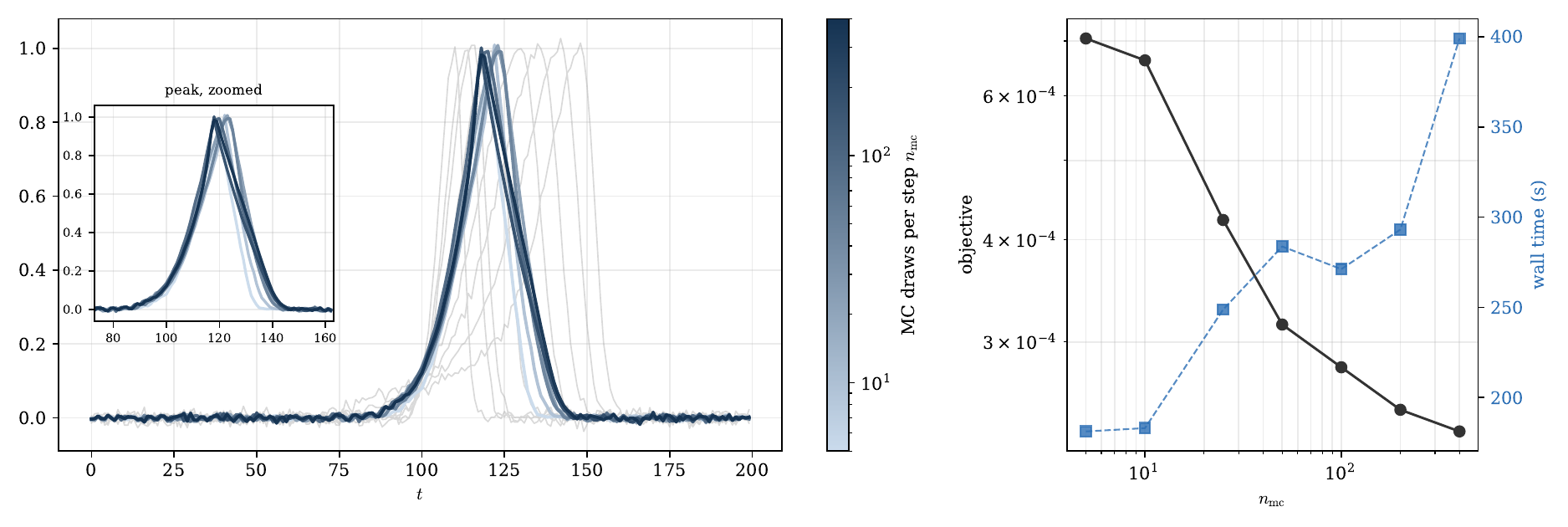}
 \caption{Effect of the Monte Carlo budget $S$ on the INR barycenter,
 \texttt{rate} (top) and \texttt{accel} (bottom), all other settings fixed.
 Left: the barycenter, from light to dark as $S$ grows.
 Right: the final objective against $S$, with wall-clock time on the right
 axis.}
 \label{fig:mc-sweep}
\end{figure}

\paragraph{Monte Carlo budget.}
The MC scheme draws $S$ times per step, so $S$ sets a trade-off between the variance of the gradient and the cost of each step. Figure~\ref{fig:mc-sweep} varies $S\in\{5,10,25,50,100,200,400\}$ for the INR barycenter \texttt{rate} and \texttt{accel}, all other settings being those given above ($30$ alternations of $500$ Adam steps, learning rate $5\times10^{-3}$, $L=10$). The left panels show the barycenters, from light to dark as $S$ grows, and the right panels the final objective against $S$, with the run time on the right axis. The objective decreases with $S$, from $5.3\times10^{-4}$ to $1.7\times10^{-4}$ on \texttt{rate} and from $7.0\times10^{-4}$ to $2.3\times10^{-4}$ on \texttt{accel}, and at $S=200$, our default, it is within $5\%$ and $7\%$ of its value at $S=400$. Over the same range, the run time grows from about $180$\,s to $400$\,s. The peak height of the barycenter seems to depend on $S$ ($0.98$ to $1.01$).

\begin{figure}[tp]
 \centering
 \includegraphics[width=0.7\columnwidth]{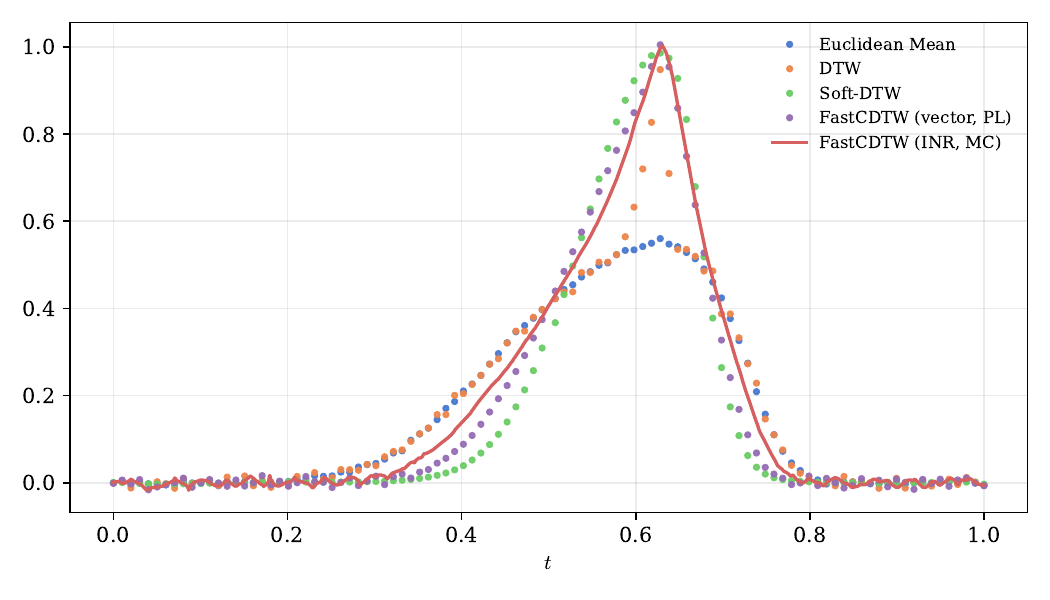}
 \includegraphics[width=0.7\columnwidth]{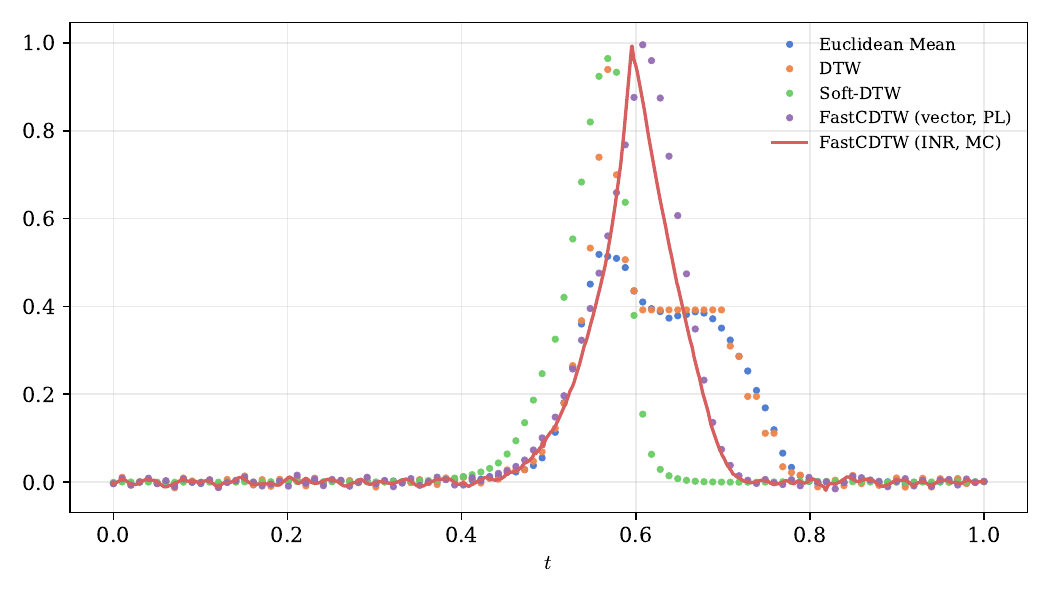}
 \includegraphics[width=0.7\columnwidth]{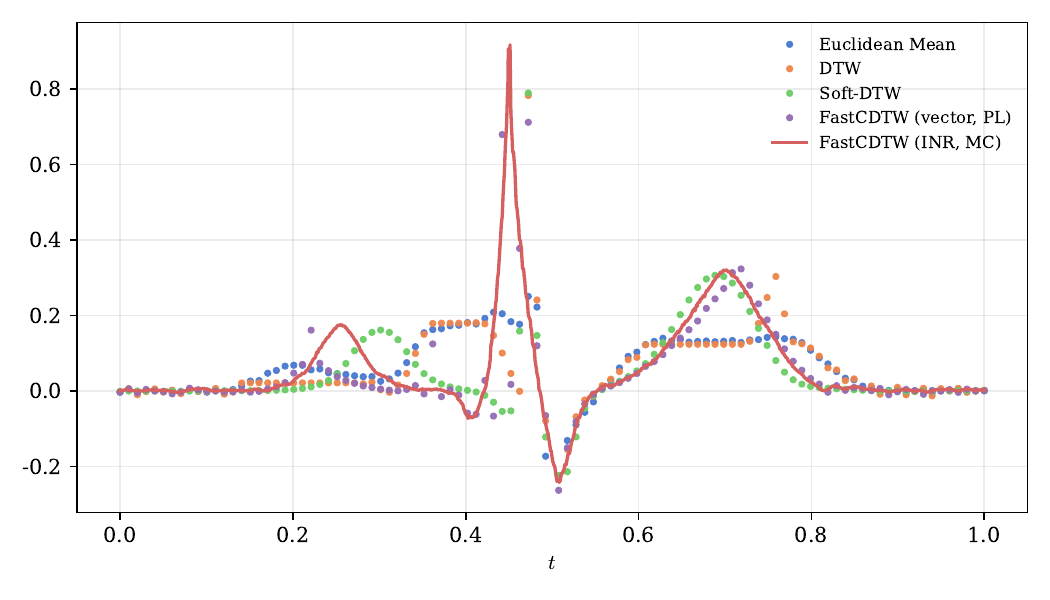}
 \caption{Coarsely sampled series ($T=100$ kept samples), \texttt{rate}
 (top), \texttt{accel} (middle) and \texttt{ecg} (bottom). The discrete
 barycenters live on the sample grid, while the INR centroid is a function. Soft-DTW is again taken at the temperature
 with the lowest DTW loss, selected on this grid.}
 \label{fig:inr-continuous}
\end{figure}

\paragraph{Coarsely sampled series.}
A vector barycenter only has values at the sampling times of the series, whereas the INR barycenter is a function of continuous time, which can be evaluated at any time once trained. To see when this matters, we keep $100$ evenly spaced samples out of the $200$ of each series, and compute the barycenters for each method on these $100$ sampling times, with Soft-DTW at the temperature with the lowest DTW loss and the INR trained with $15$ alternations and $S=200$. In Figure~\ref{fig:inr-continuous}, the barycenters of the other methods are drawn as points at the $100$ sampling times, and the INR as a curve evaluated at $2000$ evenly spaced times. On \texttt{rate} and \texttt{accel}, the peak spans several sampling times, so these $100$ values are enough to describe it, and the INR matches the vector FastCDTW barycenter (peak $1.00$ against $1.01$ on \texttt{rate}, $0.99$ against $1.00$ on \texttt{accel}). On \texttt{ecg}, the central peak is about one sampling interval wide, so the highest sample of each series ranges from $0.80$ to $1.00$ (mean $0.92$). At the sampling times, all barycenters stay at or below $0.79$, including the INR ($0.73$), but the INR reaches $0.92$ between two of them.

\begin{figure}[t]
 \centering
 \includegraphics[width=\columnwidth]{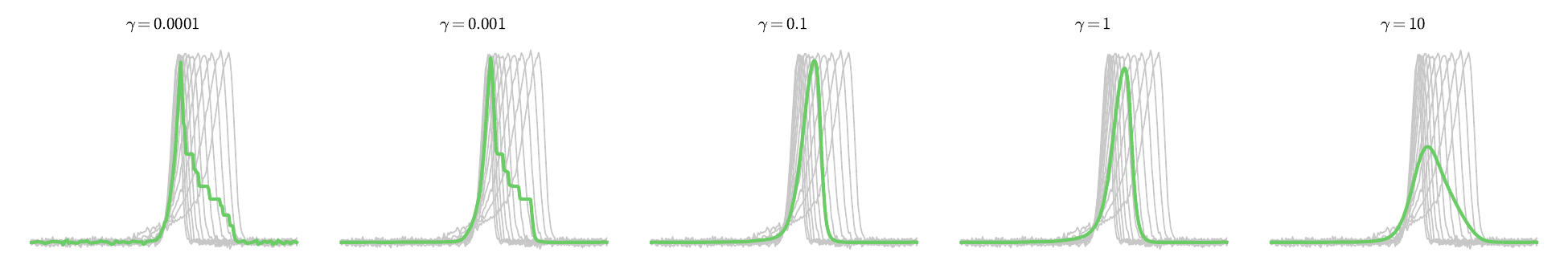}
 \includegraphics[width=\columnwidth]{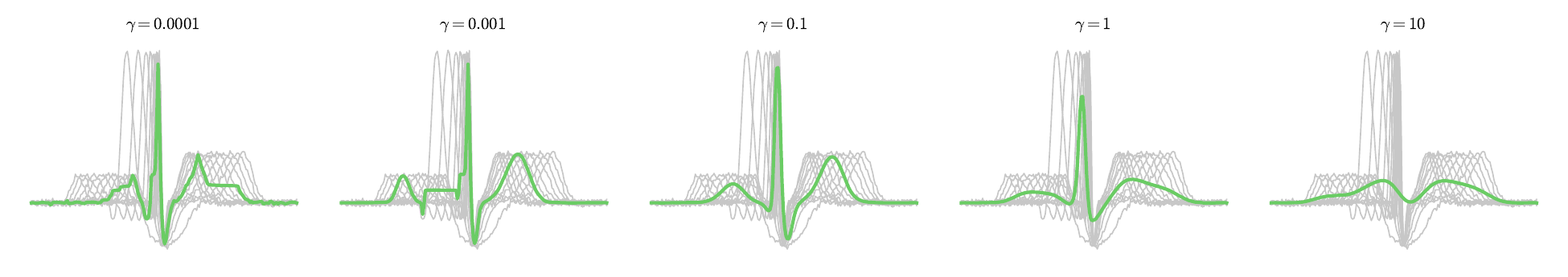}
 \includegraphics[width=\columnwidth]{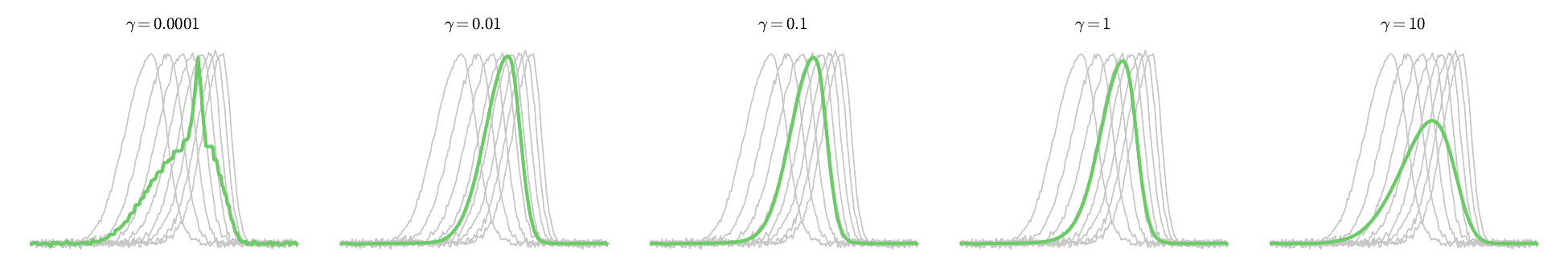}
 \caption{Soft-DTW barycenter at the temperatures not shown in
 Figure~\ref{fig:barycenters-grid}, on \texttt{accel} (top), \texttt{ecg}
 (middle) and \texttt{rate} (bottom), on the same series as that figure.}
 \label{fig:gamma-sweep}
\end{figure}

%% file: sections/appendix/nn.tex
\subsection{Time series classification}\label{app:nn}
\paragraph{Hyperparameters.}
Series are used as provided by the archive, without further $z$-normalization. DTW and Soft-DTW use the implementations of \texttt{tslearn}~\citep{tslearn}. The Soft-DTW temperature is selected in $\{10^{-4},10^{-3},\dots,10^{4}\}$ by the 1-NN accuracy averaged over $5$ random splits of the training set, with $2/3$ of it as reference series and $1/3$ as queries. For both FastCDTW schemes, the warping of each (test, train) pair uses the grid parametrization of Section~\ref{sec:param} with $N_\phi=T-1$, starts from the identity ($\theta=0$), and is fitted by $n_\text{iter}=300$ Adam steps, the same budget for every pair. The learning rate is selected in $\{10^{-3},5\times10^{-3},10^{-2}\}$ with the same protocol as the Soft-DTW temperature. The MC scheme draws $S=800$ samples per step.

\begin{table}[tp]
 \centering
 \small
 \begin{tabular}{lrr}
 \toprule
 method & CPU (s) & GPU (s) \\
 \midrule
 Euclidean & 0.01 & -- \\
 DTW & 27.3 & -- \\
 Soft-DTW & 199.3 & -- \\
 FastCDTW-PL & 2114.6 & 20.1 \\
 FastCDTW-MC & 2484.9 & 15.2 \\
 \bottomrule
 \end{tabular}
 \caption{Cost of one cross-distance matrix, in seconds, averaged over the datasets on which every method was timed on CPU, excluding hyperparameter selection.}
 \label{tab:nn-time}
\end{table}

\paragraph{Hardware.}
FastCDTW runs on GPU, with one NVIDIA A40 (48~GB), both for the accuracies and for the GPU column of Table~\ref{tab:nn-time}. The Euclidean distance, DTW and Soft-DTW run on CPU. The CPU column of Table~\ref{tab:nn-time} is measured on a single core; these jobs ran on several machines, the fastest per core being an Intel Xeon Gold 6254 CPU at 3.10~GHz.

\paragraph{Timing protocol.}
The timings of Table~\ref{tab:nn-time} cover one final cross-distance matrix:
$\gamma$ and the learning rate are read from the selection already made on the
training set, so no cross-validation is counted. All GPU rows were
measured on one NVIDIA A40.

\paragraph{Per-dataset accuracies.}
Tables~\ref{tab:nn-full-a} and \ref{tab:nn-full-b} give the 1-NN test
accuracy of every method on each of the 95 datasets. Bold marks every
method whose accuracy, as printed, is at least $99\%$ of the best accuracy on that dataset.

\begin{table}[tp]
    \centering
    \footnotesize
    \setlength{\tabcolsep}{4pt}
    \begin{tabular}{lrrrrr}
        \toprule
        dataset & Euc. & DTW & Soft-DTW & FastCDTW-PL & FastCDTW-MC \\
        \midrule
        ACSF1 & 0.540 & \textbf{0.640} & \textbf{0.640} & 0.520 & 0.530 \\
        Adiac & 0.611 & 0.604 & 0.598 & \textbf{0.693} & \textbf{0.693} \\
        ArrowHead & 0.800 & 0.703 & 0.794 & \textbf{0.823} & 0.806 \\
        BME & 0.827 & 0.893 & \textbf{0.993} & \textbf{0.993} & \textbf{0.993} \\
        Beef & \textbf{0.667} & 0.633 & 0.633 & \textbf{0.667} & \textbf{0.667} \\
        BeetleFly & \textbf{0.750} & 0.700 & 0.700 & 0.700 & 0.700 \\
        BirdChicken & 0.550 & \textbf{0.750} & \textbf{0.750} & \textbf{0.750} & \textbf{0.750} \\
        CBF & 0.852 & \textbf{0.997} & \textbf{0.997} & \textbf{0.996} & \textbf{1.000} \\
        Car & 0.733 & 0.733 & 0.750 & \textbf{0.800} & 0.783 \\
        Chinatown & 0.945 & \textbf{0.974} & \textbf{0.974} & 0.948 & 0.948 \\
        ChlorineConcentration & \textbf{0.650} & \textbf{0.648} & 0.646 & \textbf{0.653} & \textbf{0.649} \\
        CinCECGTorso & 0.897 & 0.651 & \textbf{0.920} & 0.893 & 0.904 \\
        Coffee & \textbf{1.000} & \textbf{1.000} & \textbf{1.000} & \textbf{1.000} & \textbf{1.000} \\
        Computers & 0.576 & \textbf{0.700} & \textbf{0.696} & 0.612 & 0.548 \\
        CricketX & 0.577 & 0.754 & \textbf{0.785} & \textbf{0.782} & 0.749 \\
        CricketY & 0.567 & 0.744 & \textbf{0.754} & 0.744 & 0.744 \\
        CricketZ & 0.587 & 0.754 & \textbf{0.782} & \textbf{0.785} & 0.762 \\
        DiatomSizeReduction & 0.935 & \textbf{0.967} & \textbf{0.967} & \textbf{0.958} & \textbf{0.958} \\
        DistalPhalanxOutlineAgeGroup & 0.626 & \textbf{0.770} & \textbf{0.770} & 0.683 & 0.683 \\
        DistalPhalanxOutlineCorrect & 0.717 & 0.717 & 0.725 & 0.739 & \textbf{0.750} \\
        DistalPhalanxTW & 0.633 & 0.590 & \textbf{0.655} & 0.612 & 0.619 \\
        ECG200 & \textbf{0.880} & 0.770 & 0.840 & 0.860 & 0.860 \\
        ECGFiveDays & 0.797 & 0.768 & \textbf{0.906} & \textbf{0.901} & \textbf{0.908} \\
        EOGHorizontalSignal & 0.442 & 0.442 & \textbf{0.550} & 0.511 & 0.486 \\
        EOGVerticalSignal & 0.362 & 0.434 & \textbf{0.489} & 0.481 & 0.483 \\
        Earthquakes & 0.712 & 0.719 & \textbf{0.777} & 0.719 & 0.763 \\
        FaceAll & 0.714 & \textbf{0.808} & 0.760 & 0.781 & 0.779 \\
        FaceFour & 0.784 & 0.830 & 0.818 & \textbf{0.852} & 0.841 \\
        FacesUCR & 0.769 & 0.905 & \textbf{0.926} & \textbf{0.926} & \textbf{0.924} \\
        FiftyWords & 0.631 & 0.690 & \textbf{0.796} & 0.785 & 0.767 \\
        Fish & 0.783 & 0.823 & \textbf{0.903} & \textbf{0.903} & 0.886 \\
        FreezerRegularTrain & 0.805 & 0.899 & \textbf{0.949} & \textbf{0.954} & \textbf{0.955} \\
        FreezerSmallTrain & 0.676 & \textbf{0.759} & 0.690 & 0.742 & 0.735 \\
        Fungi & 0.839 & 0.796 & 0.796 & \textbf{0.984} & \textbf{0.989} \\
        GunPoint & 0.913 & 0.907 & \textbf{0.980} & 0.973 & \textbf{0.987} \\
        GunPointAgeSpan & 0.968 & 0.984 & \textbf{0.997} & \textbf{0.994} & \textbf{0.994} \\
        GunPointMaleVersusFemale & \textbf{0.994} & 0.984 & \textbf{1.000} & \textbf{1.000} & \textbf{1.000} \\
        GunPointOldVersusYoung & \textbf{1.000} & \textbf{1.000} & \textbf{1.000} & \textbf{1.000} & \textbf{1.000} \\
        Ham & \textbf{0.600} & 0.467 & 0.581 & 0.590 & 0.581 \\
        Haptics & 0.370 & 0.377 & \textbf{0.477} & 0.406 & 0.393 \\
        Herring & 0.516 & 0.531 & \textbf{0.594} & 0.562 & \textbf{0.594} \\
        HouseTwenty & 0.681 & \textbf{0.840} & \textbf{0.840} & 0.790 & 0.782 \\
        InlineSkate & 0.342 & 0.384 & \textbf{0.433} & 0.398 & 0.416 \\
        InsectEPGRegularTrain & \textbf{1.000} & \textbf{1.000} & \textbf{1.000} & \textbf{1.000} & \textbf{1.000} \\
        InsectEPGSmallTrain & \textbf{1.000} & \textbf{1.000} & \textbf{1.000} & \textbf{1.000} & \textbf{1.000} \\
        InsectWingbeatSound & \textbf{0.562} & 0.355 & \textbf{0.559} & 0.551 & 0.546 \\
        ItalyPowerDemand & \textbf{0.955} & 0.950 & 0.948 & \textbf{0.962} & \textbf{0.962} \\
        LargeKitchenAppliances & 0.493 & \textbf{0.795} & \textbf{0.797} & 0.667 & 0.683 \\
        Lightning2 & 0.754 & 0.869 & \textbf{0.902} & 0.836 & 0.869 \\
        Lightning7 & 0.575 & 0.726 & 0.726 & \textbf{0.781} & 0.753 \\
        Mallat & 0.914 & \textbf{0.934} & \textbf{0.934} & \textbf{0.932} & \textbf{0.931} \\
        Meat & \textbf{0.933} & \textbf{0.933} & \textbf{0.933} & \textbf{0.933} & \textbf{0.933} \\
        MedicalImages & 0.684 & 0.737 & \textbf{0.751} & \textbf{0.758} & \textbf{0.751} \\
        MiddlePhalanxOutlineAgeGroup & \textbf{0.519} & 0.500 & 0.494 & 0.494 & 0.481 \\
        MiddlePhalanxOutlineCorrect & \textbf{0.766} & 0.698 & 0.722 & 0.746 & 0.732 \\
        MiddlePhalanxTW & 0.513 & 0.506 & 0.519 & \textbf{0.558} & 0.552 \\
        MixedShapesRegularTrain & 0.897 & 0.842 & \textbf{0.921} & \textbf{0.917} & \textbf{0.917} \\
        MixedShapesSmallTrain & 0.835 & 0.780 & 0.780 & \textbf{0.869} & \textbf{0.871} \\
        MoteStrain & 0.879 & 0.835 & \textbf{0.920} & 0.886 & 0.895 \\
        OSULeaf & 0.521 & 0.591 & \textbf{0.649} & 0.607 & 0.620 \\
        OliveOil & \textbf{0.867} & 0.833 & \textbf{0.867} & \textbf{0.867} & \textbf{0.867} \\
        PhalangesOutlinesCorrect & 0.761 & 0.728 & 0.719 & \textbf{0.766} & \textbf{0.769} \\
        Phoneme & 0.109 & 0.228 & \textbf{0.242} & 0.157 & 0.128 \\
        PigAirwayPressure & 0.091 & 0.183 & \textbf{0.197} & 0.159 & 0.163 \\
        \bottomrule
    \end{tabular}
    \caption{1-NN test accuracy per dataset (1 of 2). Bold: within $99\%$ of the best method on that row.}
    \label{tab:nn-full-a}
\end{table}

\begin{table}[tp]
    \centering
    \footnotesize
    \setlength{\tabcolsep}{4pt}
    \begin{tabular}{lrrrrr}
        \toprule
        dataset & Euc. & DTW & Soft-DTW & FastCDTW-PL & FastCDTW-MC \\
        \midrule
        PigArtPressure & 0.288 & 0.481 & 0.553 & 0.558 & \textbf{0.572} \\
        PigCVP & 0.139 & 0.327 & \textbf{0.418} & 0.183 & 0.188 \\
        Plane & 0.962 & \textbf{1.000} & \textbf{1.000} & \textbf{1.000} & \textbf{1.000} \\
        PowerCons & \textbf{0.978} & 0.922 & \textbf{0.972} & 0.967 & 0.967 \\
        ProximalPhalanxOutlineAgeGroup & 0.785 & \textbf{0.805} & 0.780 & \textbf{0.805} & \textbf{0.810} \\
        ProximalPhalanxOutlineCorrect & 0.808 & 0.784 & 0.725 & \textbf{0.842} & \textbf{0.845} \\
        ProximalPhalanxTW & 0.707 & 0.756 & \textbf{0.761} & \textbf{0.766} & \textbf{0.766} \\
        RefrigerationDevices & 0.395 & \textbf{0.464} & 0.456 & 0.405 & 0.405 \\
        Rock & \textbf{0.640} & 0.540 & 0.540 & 0.540 & 0.540 \\
        ScreenType & 0.360 & \textbf{0.397} & 0.376 & 0.381 & 0.365 \\
        SemgHandGenderCh2 & 0.898 & 0.915 & 0.915 & \textbf{0.943} & 0.930 \\
        ShapeletSim & 0.539 & 0.650 & 0.650 & 0.639 & \textbf{0.672} \\
        ShapesAll & 0.752 & 0.768 & \textbf{0.827} & 0.800 & 0.802 \\
        SmallKitchenAppliances & 0.341 & 0.643 & \textbf{0.667} & 0.459 & 0.501 \\
        SmoothSubspace & 0.953 & 0.827 & 0.887 & 0.953 & \textbf{0.973} \\
        SonyAIBORobotSurface1 & 0.696 & \textbf{0.725} & \textbf{0.725} & \textbf{0.727} & 0.710 \\
        SonyAIBORobotSurface2 & \textbf{0.859} & 0.831 & 0.836 & \textbf{0.855} & \textbf{0.855} \\
        Strawberry & \textbf{0.946} & \textbf{0.941} & \textbf{0.946} & \textbf{0.943} & \textbf{0.943} \\
        SwedishLeaf & 0.789 & 0.792 & 0.810 & \textbf{0.870} & \textbf{0.867} \\
        Symbols & 0.899 & 0.950 & \textbf{0.968} & 0.949 & 0.948 \\
        SyntheticControl & 0.880 & \textbf{0.993} & \textbf{0.990} & \textbf{0.987} & \textbf{0.987} \\
        ToeSegmentation1 & 0.680 & 0.772 & \textbf{0.842} & 0.811 & 0.820 \\
        ToeSegmentation2 & 0.808 & 0.838 & 0.838 & 0.869 & \textbf{0.900} \\
        Trace & 0.760 & \textbf{1.000} & \textbf{1.000} & \textbf{0.990} & \textbf{0.990} \\
        TwoLeadECG & 0.747 & \textbf{0.904} & \textbf{0.905} & 0.874 & 0.880 \\
        UMD & 0.806 & 0.882 & 0.882 & \textbf{0.986} & \textbf{0.986} \\
        Wine & 0.611 & 0.574 & 0.556 & \textbf{0.704} & \textbf{0.704} \\
        WordSynonyms & 0.618 & 0.649 & \textbf{0.766} & 0.757 & 0.755 \\
        Worms & 0.455 & \textbf{0.584} & 0.532 & 0.545 & 0.545 \\
        WormsTwoClass & 0.610 & 0.623 & \textbf{0.649} & \textbf{0.649} & 0.636 \\
        Yoga & 0.830 & 0.836 & \textbf{0.853} & \textbf{0.859} & \textbf{0.858} \\
        \midrule
        mean & 0.703 & 0.739 & 0.767 & 0.763 & 0.763 \\
        \bottomrule
    \end{tabular}
    \caption{1-NN test accuracy per dataset (3 of 3), with the mean over all 95 datasets. Bold: within $99\%$ of the best method on that row.}
    \label{tab:nn-full-b}
\end{table}

%% file: sections/appendix/domain_adaptation.tex
\subsection{Domain adaptation}\label{app:mad}

\paragraph{Hyperparameters.}
We follow the experimental protocol of \citet{mad}, and only detail below the settings needed to reproduce our results.
The backbone classifier is a fully convolutional network made of three 1D convolutional blocks with $128$, $256$ and $128$ filters of size $8$, $5$ and $3$ respectively, each followed by batch normalization and a ReLU activation, then global average pooling over time and a linear classification layer, as in \citet{mad}.
It is trained with Adam, with learning rate $10^{-4}$ and batch size $256$ (class-balanced mini-batches drawn from each domain), for $3{,}000$ iterations.
The MAD loss weights its temporal alignment (feature transport) and label transport (cross-entropy) terms with $\alpha=0.01$ and $\beta=0.01$, the values used by \citet{mad} with DTW.
With FastCDTW, the alignment cost is an integral over $[0,1]$ rather than a sum over the matched pairs of a DTW path, so its scale differs by a factor of the order of the series length; we therefore rescale $\alpha$ to $\alpha\,\bar{T}$, with $\bar{T}=\lfloor (T_s+T_t)/2 \rfloor$ the average of the source and target series lengths, for both FastCDTW variants.
The shared warping $\phi$ uses the grid parametrization of Section~\ref{sec:param} with $N_\phi=\bar{T}$, is initialized to the identity ($\theta=0$), and is trained jointly with the backbone, with the same optimizer and learning rate as the backbone.
The MC variant draws $S=50$ samples per loss evaluation.

\paragraph{Other parametrizations of $\phi$.}
Besides the grid parametrization of Section~\ref{sec:param}, we tried two smooth parametrizations of $\phi$. Both lie in $\mathcal{C}$ whatever the value of $\theta$.
\begin{itemize}
  \item \emph{Integrated MLP.} $\phi(t)=\int_0^t g_\theta/\int_0^1 g_\theta$ with $g_\theta$ a small MLP whose last layer is a SoftPlus, so that $g_\theta>0$ for every $\theta$ and $\dot\phi=g_\theta/\int_0^1 g_\theta$ in closed form. Smooth rather than piecewise-affine, so evaluated with MC.
  \item \emph{INR integrated MLP.} The same construction with $g_\theta$ applied to a Fourier encoding $\gamma(t)=(\sin(2\pi f_l t),\cos(2\pi f_l t))_{l\leq L}$ of $t$, to represent high-frequency warpings; MC as well. The frequencies follow the same geometric rule as for the INR barycenter of Appendix~\ref{app:inr}, $f_l=2^{l-2}$ for $l=1,\dots,L$, here with $L=8$, so they range from $1/2$ to $64$. The MLP $g_\theta$ has $4$ hidden layers of width $64$ with ReLU activations, followed by the SoftPlus output layer.
\end{itemize}

\paragraph{Impact of the warping parametrization $\phi$.}
Table~\ref{tab:generator_impact} compares these parametrizations of $\phi$ within FastCDTW-MC.
The grid parametrization is the strongest and most robust choice, matching or slightly exceeding DTW's accuracy on both datasets.
The INR integrated MLP stays close to DTW's accuracy on both datasets. %
The Integrated MLP, which lacks this high-frequency Fourier encoding, lags further behind on both datasets, underlining the benefit of Fourier features for representing warpings competitive with DTW's.
This validates our default choice of grid parametrization for the other experiments.

\begin{table}
  \centering
  \begin{tabular}{lcc}
    \toprule
    Warping ($\phi$) & TarnBrittany & miniTimeMatch \\
    \midrule
    Grid & $\boldsymbol{94.4 \pm 5.8}$ & $\boldsymbol{84.5 \pm 7.2}$ \\
    INR integrated MLP & $93.1 \pm 7.5$ & $79.3 \pm 15.2$ \\
    Integrated MLP & $92.0 \pm 9.4$ & $76.6 \pm 9.8$ \\
    \midrule
    DTW & $94.0 \pm 6.4$ & $78.9 \pm 8.0$ \\
    \bottomrule
  \end{tabular}
  \caption{Impact of the warping parametrization ($\phi$) on FastCDTW-MC accuracy (mean over 3 seeds and over problem pairs within each dataset; std over problem pairs only).}
  \label{tab:generator_impact}
\end{table}

%% file: biblio.bib
@inproceedings{Wehenkel2019UMNN,
 author = {Wehenkel, Antoine and Louppe, Gilles},
 booktitle = {Advances in Neural Information Processing Systems},
 editor = {H. Wallach and H. Larochelle and A. Beygelzimer and F. d\textquotesingle Alch\'{e}-Buc and E. Fox and R. Garnett},
 pages = {},
 publisher = {Curran Associates, Inc.},
 title = {Unconstrained Monotonic Neural Networks},
 url = {https://proceedings.neurips.cc/paper_files/paper/2019/file/2a084e55c87b1ebcdaad1f62fdbbac8e-Paper.pdf},
 volume = {32},
 year = {2019},
 url={https://papers.nips.cc/paper_files/paper/2019/hash/2a084e55c87b1ebcdaad1f62fdbbac8e-Abstract.html}
}

@inproceedings{runje2023constrained,
  title = 	 {Constrained Monotonic Neural Networks},
  author =       {Runje, Davor and Shankaranarayana, Sharath M},
  booktitle = 	 {Proceedings of the 40th International Conference on Machine Learning},
  pages = 	 {29338--29353},
  year = 	 {2023},
  editor = 	 {Krause, Andreas and Brunskill, Emma and Cho, Kyunghyun and Engelhardt, Barbara and Sabato, Sivan and Scarlett, Jonathan},
  volume = 	 {202},
  series = 	 {Proceedings of Machine Learning Research},
  month = 	 {23--29 Jul},
  publisher =    {PMLR},
  url = 	 {https://proceedings.mlr.press/v202/runje23a.html}
}

@article{Vintsyuk1968,
  author       = {Vintsyuk, T. K.},
  title        = {Speech Discrimination by Dynamic Programming},
  journal      = {Cybernetics},
  volume       = {4},
  number       = {1},
  pages        = {52--57},
  year         = {1968},
  url = {https://doi.org/10.1007/BF01074755}
}

@article{SakoeChiba1978,
  author={Sakoe, H. and Chiba, S.},
  journal={IEEE Transactions on Acoustics, Speech, and Signal Processing}, 
  title={Dynamic programming algorithm optimization for spoken word recognition}, 
  year={1978},
  url = {https://doi.org/10.1109/TASSP.1978.1163055},
  volume={26},
  number={1},
  pages={43-49},
  doi={10.1109/TASSP.1978.1163055}
}

@InProceedings{softdtw,
  title = 	 {Soft-{DTW}: a Differentiable Loss Function for Time-Series},
  author =       {Marco Cuturi and Mathieu Blondel},
  booktitle = 	 {Proceedings of the 34th International Conference on Machine Learning},
  pages = 	 {894--903},
  year = 	 {2017},
  editor = 	 {Precup, Doina and Teh, Yee Whye},
  volume = 	 {70},
  series = 	 {Proceedings of Machine Learning Research},
  month = 	 {06--11 Aug},
  publisher =    {PMLR},
  url = 	 {https://proceedings.mlr.press/v70/cuturi17a.html}
}

@article{gdtw,
  title = {A General Optimization Framework for Dynamic Time Warping},
  author = {Deriso, Dave and Boyd, Stephen},
  year = {2023},
  journal = {Optimization and Engineering},
  volume = {24},
  number = {2},
  pages = {1411-1432},
  publisher = {{Springer Science and Business Media LLC}},
  url = {https://doi.org/10.1007/s11081-022-09738-z},
}

@inproceedings{mad,
  TITLE = {{Match-And-Deform: Time Series Domain Adaptation through Optimal Transport and Temporal Alignment}},
  AUTHOR = {Painblanc, Fran{\c c}ois and Chapel, Laetitia and Courty, Nicolas and Friguet, Chlo{\'e} and Pelletier, Charlotte and Tavenard, Romain},
  BOOKTITLE = {{ECML PKDD 2023}},
  PUBLISHER = {{Springer Nature Switzerland}},
  SERIES = {Lecture Notes in Computer Science},
  VOLUME = {14173},
  PAGES = {341-356},
  YEAR = {2023},
  url={https://doi.org/10.1007/978-3-031-43424-2_21},
}

@article{buchin2025cdtw,
  title={Computing continuous dynamic time warping of time series in polynomial time}, 
  volume={16}, 
  url={https://jocg.org/index.php/jocg/article/view/4022}, DOI={10.20382/jocg.v16i1a21}, 
  number={1}, 
  journal={Journal of Computational Geometry}, 
  author={Buchin, Kevin and Nusser, André and Wong, Sampson}, 
  year={2025}, 
  month={Dec.}, 
  pages={765–799},
  url = {https://doi.org/10.20382/jocg.v16i1a21},
}

@phdthesis{buchin2007frechet,
  author = {Buchin, Maike Elisabeth},
  year = {2007},
  title = {On the Computability of the Frechet Distance Between Triangulated Surfaces},
  school = {Freie Universit{\"a}t Berlin},
  url = "http://dx.doi.org/10.17169/refubium-6111",
}

@article{ergur2026cdtw,
  title = {Approximating {CDTW} Distance of Piecewise Algebraic Curves},
  author = {Erg{\"u}r, Alperen A. and Khowala, Shamik},
  year = {2026},
  journal = {arXiv preprint arXiv:2609.04294},
  url = {https://arxiv.org/abs/2609.04294},
}

@article{dba,
  title = {A global averaging method for dynamic time warping, with applications to clustering},
  journal = {Pattern Recognition},
  volume = {44},
  number = {3},
  pages = {678-693},
  year = {2011},
  issn = {0031-3203},
  doi = {https://doi.org/10.1016/j.patcog.2010.09.013},
  url = {https://www.sciencedirect.com/science/article/pii/S003132031000453X},
  author = {François Petitjean and Alain Ketterlin and Pierre Gan{\c c}arski},
}

@InProceedings{softdtwdiv,
  title = 	 { Differentiable Divergences Between Time Series },
  author =       {Blondel, Mathieu and Mensch, Arthur and Vert, Jean-Philippe},
  booktitle = 	 {Proceedings of The 24th International Conference on Artificial Intelligence and Statistics},
  pages = 	 {3853--3861},
  year = 	 {2021},
  editor = 	 {Banerjee, Arindam and Fukumizu, Kenji},
  volume = 	 {130},
  series = 	 {Proceedings of Machine Learning Research},
  month = 	 {13--15 Apr},
  publisher =    {PMLR},
  url = 	 {https://proceedings.mlr.press/v130/blondel21a.html},
}

@InProceedings{rhodes2024timewarpvaesimultaneoustimewarpingrepresentation,
      title={{TimewarpVAE}: Simultaneous Time-Warping and Representation Learning of Trajectories},
      author={Travers Rhodes and Daniel D. Lee},
      year={2024},
      booktitle = {IJCAI'24 Workshop on Artificial Intelligence for Time Series Analysis (AI4TS)},
}

@article{bagnall2017bakeoff,
  title = {The Great Time Series Classification Bake Off: a Review and Experimental Evaluation of Recent Algorithmic Advances},
  author = {Bagnall, Anthony and Lines, Jason and Bostrom, Aaron and Large, James and Keogh, Eamonn},
  year = {2017},
  journal = {Data Mining and Knowledge Discovery},
  volume = {31},
  number = {3},
  pages = {606-660},
  publisher = {Springer},
  url={https://doi.org/10.1007/s10618-016-0483-9},
}

@inproceedings{lods2017ldps,
  title = {Learning {DTW}-Preserving Shapelets},
  author = {Lods, Arnaud and Malinowski, Simon and Tavenard, Romain and Amsaleg, Laurent},
  booktitle = {Advances in Intelligent Data Analysis XVI (IDA 2017)},
  series = {Lecture Notes in Computer Science},
  volume = {10584},
  pages = {198-209},
  year = {2017},
  publisher = {Springer},
  url = {https://doi.org/10.1007/978-3-319-68765-0_17}
}

@inproceedings{rubanova2019latentode,
  author = {Rubanova, Yulia and Chen, Ricky T. Q. and Duvenaud, David K},
  booktitle = {Advances in Neural Information Processing Systems},
  editor = {H. Wallach and H. Larochelle and A. Beygelzimer and F. d\textquotesingle Alch\'{e}-Buc and E. Fox and R. Garnett},
  pages = {},
  publisher = {Curran Associates, Inc.},
  title = {Latent Ordinary Differential Equations for Irregularly-Sampled Time Series},
  url = {https://proceedings.neurips.cc/paper_files/paper/2019/file/42a6845a557bef704ad8ac9cb4461d43-Paper.pdf},
  volume = {32},
  year = {2019}
}

@article{UCRArchive,
  author={Dau, Hoang Anh and Bagnall, Anthony and Kamgar, Kaveh and Yeh, Chin-Chia Michael and Zhu, Yan and Gharghabi, Shaghayegh and Ratanamahatana, Chotirat Ann and Keogh, Eamonn},
  journal={IEEE/CAA Journal of Automatica Sinica}, 
  title={The UCR time series archive}, 
  year={2019},
  volume={6},
  number={6},
  pages={1293-1305},
  doi={10.1109/JAS.2019.1911747},
  url={https://doi.org/10.1109/JAS.2019.1911747}
}

@article{naour2024time,
title={Time Series Continuous Modeling for Imputation and Forecasting with Implicit Neural Representations},
author={Etienne Le Naour and Louis Serrano and L{\'e}on Migus and Yuan Yin and Ghislain Agoua and Nicolas Baskiotis and Patrick Gallinari and Vincent Guigue},
journal={Transactions on Machine Learning Research},
issn={2835-8856},
year={2024},
url={https://openreview.net/forum?id=P1vzXDklar},
note={}
}

@article{tslearn,
  author  = {Romain Tavenard and Johann Faouzi and Gilles Vandewiele and Felix Divo and Guillaume Androz and Chester Holtz and Marie Payne and Roman Yurchak and Marc Ru{\ss}wurm and Kushal Kolar and Eli Woods},
  title   = {Tslearn, A Machine Learning Toolkit for Time Series Data},
  journal = {Journal of Machine Learning Research},
  year    = {2020},
  volume  = {21},
  number  = {118},
  pages   = {1--6},
  url     = {http://jmlr.org/papers/v21/20-091.html}
}
